\documentclass[letterpaper]{article} 
\usepackage{aaai2026}  
\usepackage{times}  
\usepackage{helvet}  
\usepackage{courier}  
\usepackage[hyphens]{url}  
\usepackage{graphicx} 
\usepackage{natbib}  
\usepackage{caption} 
\usepackage{algorithm}
\usepackage{algorithmic}

\usepackage{newfloat}
\usepackage{listings}
\DeclareCaptionStyle{ruled}{labelfont=normalfont,labelsep=colon,strut=off} 
\floatstyle{ruled}
\newfloat{listing}{tb}{lst}{}
\floatname{listing}{Listing}
\usepackage{subcaption}
\usepackage{amsmath}
\usepackage{amsthm}
\usepackage{amssymb}
\usepackage{bm}
\usepackage{booktabs}
\usepackage[cmyk,table,xcdraw]{xcolor}

\newtheorem{theorem}{Theorem}

\title{Bridging Modalities via Progressive Re-alignment for \\Multimodal Test-Time Adaptation}
\author{
    Jiacheng Li\textsuperscript{\rm 1, \rm 3}, Songhe Feng\textsuperscript{\rm 2, \rm 3}\thanks{Corresponding Author.}
}
\affiliations{
    \textsuperscript{\rm 1}Key Laboratory of Big Data \& Artificial Intelligence in Transportation (Beijing Jiaotong University), \\Ministry of Education, China\\
    \textsuperscript{\rm 2}Tangshan Research Institute, Beijing Jiaotong University, China\\
    \textsuperscript{\rm 3}School of Computer Science and Technology, Beijing Jiaotong University, Beijing, China\\

    jiacheng.li@bjtu.edu.cn, shfeng@bjtu.edu.cn
}

\usepackage{bibentry}

\begin{document}

\maketitle

\begin{abstract}
Test-time adaptation (TTA) enables online model adaptation using only unlabeled test data, aiming to bridge the gap between source and target distributions. However, in multimodal scenarios, varying degrees of distribution shift across different modalities give rise to a complex coupling effect of unimodal shallow feature shift and cross-modal high-level semantic misalignment, posing a major obstacle to extending existing TTA methods to the multimodal field. To address this challenge, we propose a novel multimodal test-time adaptation (MMTTA) framework, termed as \textbf{Bri}dging \textbf{M}odalities via \textbf{P}rogressive \textbf{R}e-alignment (\textbf{BriMPR}). BriMPR, consisting of two progressively enhanced modules, tackles the coupling effect with a divide-and-conquer strategy. Specifically, we first decompose MMTTA into multiple unimodal feature alignment sub-problems. By leveraging the strong function approximation ability of prompt tuning, we calibrate the unimodal global feature distributions to their respective source distributions, so as to achieve the initial semantic re-alignment across modalities. Subsequently, we assign the credible pseudo-labels to combinations of masked and complete modalities, and introduce inter-modal instance-wise contrastive learning to further enhance the information interaction among modalities and refine the alignment. Extensive experiments on MMTTA tasks, including both corruption-based and real-world domain shift benchmarks, demonstrate the superiority of our method.
\end{abstract}

\begin{links}
    \link{Code}{https://github.com/Luchicken/BriMPR}
\end{links}

\section{Introduction}
Despite the remarkable success of deep neural networks in various fields, their excellent performances often hinge on specific data conditions. The possible distribution shift (or domain shift) between training and testing data has become a major obstacle to model generalization. Unsupervised domain adaptation (UDA)~\cite{DDC, DAN, MCC} and domain generalization (DG)~\cite{MixStyle, MLDG, 10.1145/3581783.3611764} have been proposed to mitigate domain gaps by designing sophisticated strategies that enable the model to adapt to the target domain during training. In contrast, test-time adaptation (TTA)~\cite{TTT, Tent, MEMO, EATA} adjusts the model according to specific test data during the test stage, reducing the dependence on the training process and training data, thereby making it a promising and more practical solution.

With the advancement of sensor technology, integrating and leveraging multimodal data collected from diverse sensors has significantly enhanced the perception capability of intelligent systems. Nevertheless, multimodal data also suffer from distribution shifts. What's worse, due to the complexity of multimodal data, different modalities often exhibit varying degrees of distribution shift from the source domain, inducing a complex coupling effect of unimodal shallow feature shift and cross-modal high-level semantic misalignment. Existing TTA methods, which are primarily designed for unimodal tasks, struggle to ensure consistent improvements across all modalities and often fail to fully exploit the rich information available in multimodal inputs. 
In Fig.~\ref{fig:tsne}, we visualize both unimodal and multimodal feature representations during the adaptation on the audio-visual event classification dataset Kinetics50–C~\cite{READ}. As a representative unimodal TTA method, EATA~\cite{EATA} reduces the uncertainty of model predictions by minimizing the entropy of reliable samples. However, it shows limited improvement in bridging the domain gap between source and target features for each modality. READ~\cite{READ}, a pioneering method for multimodal test-time adaptation (MMTTA), adapts the model by updating the self-attention layers in the fusion module to assign more weights to  the high-quality modality. Nevertheless, it lacks the correction of shallow unimodal features. 
As shown in Fig.~\ref{fig:tsne_EATA} and Fig.~\ref{fig:tsne_READ}, the lack of effective guidance for unimodal features hinders proper alignment across modalities. As a result, the fused multimodal feature representations derived from multiple unimodal features become entangled, leading to a significant decline in discriminability.

\begin{figure*}[t]
\centering

\begin{subfigure}{0.286\textwidth}
\includegraphics[width=1\textwidth]{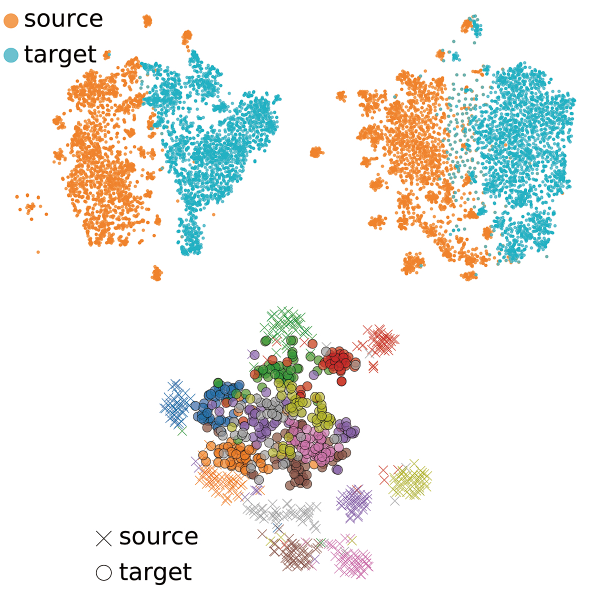}
\caption{EATA}
\label{fig:tsne_EATA}
\end{subfigure}
\hspace{3mm}
\begin{subfigure}{0.286\textwidth}
\includegraphics[width=1\textwidth]{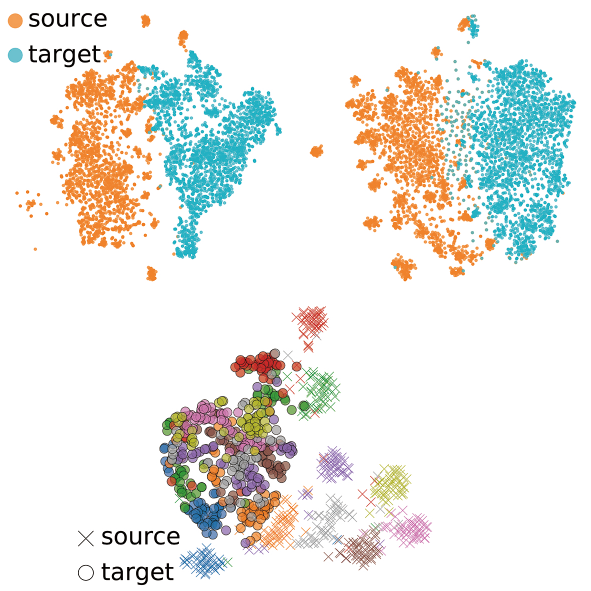}
\caption{READ}
\label{fig:tsne_READ}
\end{subfigure}
\hspace{3mm}
\begin{subfigure}{0.286\textwidth}
\includegraphics[width=1\textwidth]{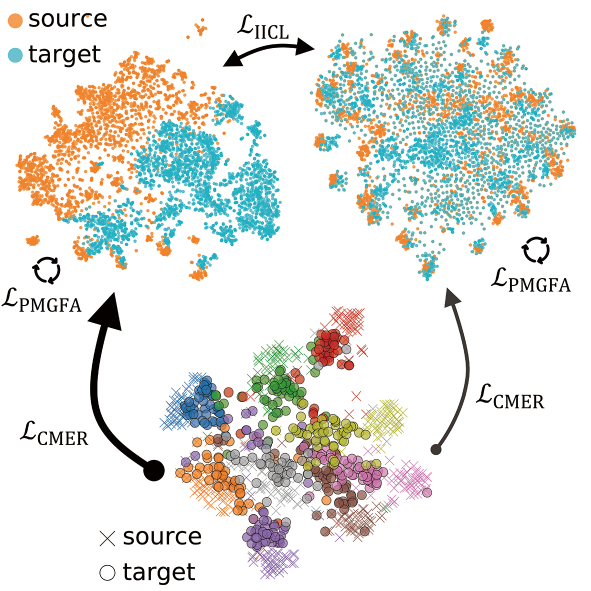}
\caption{BriMPR}
\label{fig:tsne_BriMPR}
\end{subfigure}
  
\caption{t-SNE visualizations of unimodal (top) and fused multimodal (bottom) features during adaptation versus source features. For fused features, 10 classes from Kinetics50–C are shown.}
\label{fig:tsne}
\end{figure*}

In this work, we propose $\textbf{Bri}$dging $\textbf{M}$odalities via $\textbf{P}$rogressive $\textbf{R}$e-alignment ($\textbf{BriMPR}$) for multimodal test-time adaptation. Through the joint efforts of self-calibration for each modality and inter-modal information interaction, BriMPR realigns the modalities that are subject to distribution shift with each other. 
Since the feature representations of each modality are well-aligned in the source space, we first decompose MMTTA into multiple unimodal feature alignment sub-problems. Leveraging the strong function approximation ability of prompt tuning~\cite{NEURIPS2023_eef6aecf}, we calibrate the global feature distribution of each modality to its corresponding source distribution via modality-specific prompts embedded across layers of the modality-specific encoders, thereby indirectly achieving initial cross-modal semantic alignment. 
Subsequently, the alignment is further refined by enhancing inter-modal information interaction. We propose a novel cross-modal masked embedding recombination loss, which promotes the extraction of multimodal information by providing calibrated pseudo-labels for the combinations of masked and complete modalities. Additionally, we introduce inter-modal instance-wise contrastive learning to maintain cross-modal alignment at the instance level. As shown in Fig.~\ref{fig:tsne_BriMPR}, BriMPR effectively bridges the domain gap between the source and target for each unimodal feature, thereby enhancing the discriminability of the fused features. Our contributions can be summarized as follows:
\begin{itemize}
\item We propose a novel MMTTA framework which mitigates modality-wise distribution shifts in a divide-and-conquer manner, facilitating the re-alignment among modalities.
\item We leverage the excellent function approximation ability of prompt tuning to achieve efficient calibration of the unimodal global feature distribution, and propose a novel cross-modal masked embedding recombination strategy to enhance the inter-modal interaction.
\item We conduct extensive experiments on MMTTA benchmarks, including corruption shift and real-world shift datasets, demonstrating the superiority of BriMPR over existing SOTA methods.
\end{itemize}

\section{Related Work}
\subsubsection{Test-Time Adaptation.}
Test-time adaptation (TTA) leverages unlabeled test data to adapt models to unseen target domains during test-time. 
The idea of TTA can be traced back to TTT~\cite{TTT}, which uses a self-supervised auxiliary branch to enable adaptation during inference. A series of works~\cite{Tent, EATA, SAR, DeYO} explore fully test-time adaptation (FTTA) by optimizing the normalization layers via entropy-based losses, without altering the pre-training stage. 
Given the limitations of unimodal TTA methods in multimodal scenarios, MM-TTA~\cite{MM-TTA} proposes a cross-modal self-learning framework for MMTTA. READ~\cite{READ} highlights the reliability bias of MMTTA under unimodal corruption, and proposes to adaptively assign modality weights by optimizing the self-attention in the fusion module. ABPEM~\cite{ABPEM} reduces the gap between cross-attention and self-attention, and computes the principal part of entropy to reduce gradient noise. SuMi~\cite{SuMi} utilizes interquartile range smoothing to identify samples used for calculating entropy loss. Moreover, AEO~\cite{AEO} introduces unseen classes and proposes the Multimodal open-set test-time adaptation setting. 
In this work, we attribute the difficulties of MMTTA to the coupling effect of unimodal shallow feature shift and cross-modal high-level semantic misalignment, and propose a divide-and-conquer method to re-bridge modalities during testing.

\subsubsection{Prompt Tuning.}
Originally developed in natural language processing, prompt tuning introduces extra tokens to guide models toward generating task-specific outputs. 
In computer vision, approaches like CoOp~\cite{CoOp} and CoCoOp~\cite{CoCoOp} leverage learnable prompts to enhance the zero-shot recognition capabilities of vision-language models (VLMs). Integrating the idea of TTA, test-time prompt tuning (TPT)~\cite{TPT, DiffTPT, HisTPT} fine-tunes text prompts using test samples to improve the generalization of VLMs. 
While TPT primarily focuses on extracting rich knowledge from large-scale VLMs, our work is more closely aligned with visual prompt tuning (VPT)~\cite{VPT, GPT}. VPT introduces prompt tuning into Vision Transformer, achieving significant performance gains over full fine-tuning. 
Our work extends prompt tuning to MMTTA tasks, leveraging the strong function approximation ability of prompts to efficiently calibrate the distribution of each unimodal feature—not limited to visual features alone.

\begin{figure*}[t]
\centering
\includegraphics[width=0.81\textwidth]{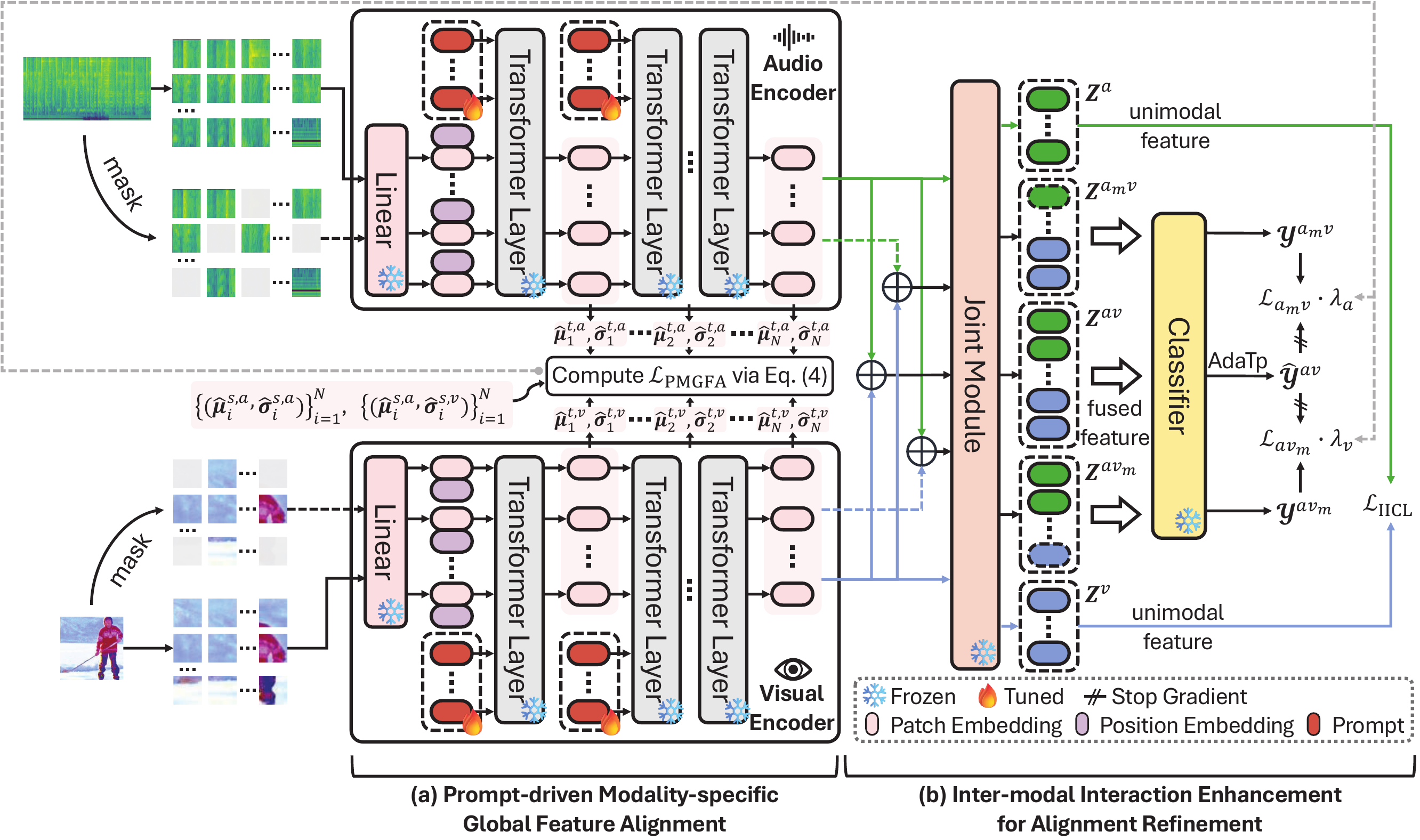}
\caption{Overview of BriMPR. BriMPR achieves initial alignment and alignment refinement through two progressive modules. The added modality-specific prompts are used to project the unimodal features into the re-aligned feature space.}
\label{fig:overview}
\end{figure*}

\section{Preliminaries}
\subsubsection{Multimodal Test-Time Adaptation (MMTTA).}
Without loss of generality, we take two modalities as an example to provide a formal definition of MMTTA. 
An off-the-shelf model $\mathcal{F}_{\Theta}$ pre-trained on the source domain $\mathcal{D}_{\mathcal{S}} = \{ (\mathbf{x}_i^{u_1}, \mathbf{x}_i^{u_2}, y_i) \}_{i=1}^{N_\mathcal{S}}$ is adopted as the initial model, where the two modalities of the source data follow the probability distributions $\mathbf{x}_i^{u_1}\sim P_{\mathcal{S}, u_1} ( \mathbf{x} )$ and $\mathbf{x}_i^{u_2}\sim P_{\mathcal{S}, u_2} ( \mathbf{x} )$, respectively. 
The goal of MMTTA is to adapt $\mathcal{F}_{\Theta}$ online to the target domain $\mathcal{D}_{\mathcal{T}} = \{ (\mathbf{x}_j^{u_1}, \mathbf{x}_j^{u_2}) \}_{j=1}^{N_\mathcal{T}}$, where the two modalities of target data follow the probability distributions $\mathbf{x}_j^{u_1}\sim P_{\mathcal{T}, u_1} ( \mathbf{x} )$ and $\mathbf{x}_j^{u_2}\sim P_{\mathcal{T}, u_2} ( \mathbf{x} )$. 
During adaptation, the source domain is inaccessible and there is a domain shift between the source and target distributions, i.e., $P_{\mathcal{S}, u_1} (\mathbf{x}) \ne P_{\mathcal{T}, u_1} (\mathbf{x})$ and $P_{\mathcal{S}, u_2} (\mathbf{x}) \ne P_{\mathcal{T}, u_2} (\mathbf{x})$.

\subsubsection{Prompt Tuning.}
Prompt tuning is regarded as a parameter-efficient fine-tuning technique, which adapts the model to downstream tasks by prepending and optimizing learnable prompt tokens into the input sequence~\cite{prefixtuning, prompttuning, VPT, ptuning}. 
For an encoder $\Phi$ consisting of $N$ transformer layers, when inserting a specified number of prompts into the input sequence at each layer, the forward process of the $i$-th layer can be formulated as: 
\begin{equation}
    \left[\_; \boldsymbol{E}_i\right] = L_{i} \left( \left[\boldsymbol{P}_{i-1}; \boldsymbol{E}_{i-1}\right] \right), \quad i=1,..., N.
    \label{eq:prompt}
\end{equation}
Here $\boldsymbol{E}_i = [\boldsymbol{e}_{i,1};\boldsymbol{e}_{i,2};\dots;\boldsymbol{e}_{i,m}]$ and $\boldsymbol{P}_i = [\boldsymbol{p}_{i,1};\boldsymbol{p}_{i,2};\dots;\allowbreak \boldsymbol{p}_{i,m_p}]$ denote the sequences of original input tokens and inserted prompt tokens, where $m$ and $m_p$ is the number of tokens, and the token dimension is $d$. $\left[\cdot; \cdot\right]$ denotes token-level concatenation. 
Then, a supervised loss $\mathcal{L}$ is minimized over the downstream dataset $\mathcal{D}_{\text{ds}}$ to obtain the optimal prompt $\boldsymbol{P}^* = \{\boldsymbol{P}_0^*, \boldsymbol{P}_1^*, \dots, \boldsymbol{P}_{N-1}^*\}$:
\begin{equation}
    \boldsymbol{P}^{*} = \underset{\boldsymbol{P}} {\operatorname{a r g \, min}} \ \mathbb{E}_{( \mathbf{x}, y ) \sim\mathcal{D}_{\text{ds}}} {\mathcal{L}} ( h ( \operatorname{MeanPool}(\boldsymbol{E}_N) ), y ),
    \label{eq:pt}
\end{equation}
where $h$ denotes the classifier. In MMTTA, due to the absence of annotation for the test data, the loss must be reformulated to enable the learning of task-specific prompts.

\section{Methodology}
\label{Methodology}
In this section, we introduce BriMPR for MMTTA, with its overall framework illustrated in Fig.~\ref{fig:overview}. 
BriMPR comprises two progressively enhanced modules: 
(a) \textit{Prompt-driven Modality-specific Global Feature Alignment} achieves initial cross-modal alignment by minimizing the discrepancy between the unimodal target statistics and their corresponding in-distribution statistics; 
(b) \textit{Inter-modal Interaction Enhancement for Alignment Refinement} further refines the alignment by providing credible pseudo-labels for combinations of masked and complete modalities, and conducting inter-modal instance-wise contrastive learning.

Following READ~\cite{READ}, we decompose the source model into two modality-specific encoders ($\Phi^a$ for the audio modality and $\Phi^v$ for the visual modality), a joint module $\Psi$, and a classifier $h$. 
We update only the prompts for each modality-specific encoder, keeping the rest of the model frozen, to recalibrate individual feature distributions and achieve bottom-up modality re-alignment.

\subsection{Prompt-driven Modality-specific Global Feature Alignment (PMGFA)}
\label{PMGFA}
The final prediction of a multimodal model comes from the joint effect of multiple individual modalities. This naturally allows MMTTA to be decomposed into multiple unimodal test-time adaptation problems. 
On the other hand, if the target representations at test time can be well projected back to the corresponding source representations, then a TTA model tends to perform well. 
Based on the intuitions above, we decouple MMTTA into multiple modality-specific feature alignment sub-problems. Since the inter-modal semantic representations are well aligned in the source representation space, solving these sub-problems means indirectly achieving cross-modal semantic alignment of the target representation.

Concretely, we first model the modality-specific source and target feature distributions as multivariate Gaussian distributions, i.e., $P_{\mathcal{S}, u} = \mathcal{N}(\mu^{s,u}, \Sigma^{s,u})$ and $P_{\mathcal{T}, u} = \mathcal{N}(\mu^{t,u}, \Sigma^{t,u})$, where $u \in \{a, v\}$. 
In prior works~\cite{TTT++, TTAC, Ada-ReAlign}, feature alignment is typically achieved by matching the first and second moments between distributions (i.e., $| | \mu^{t} - \mu^{s} | |_{2}^{2} + | | \Sigma^{t} - \Sigma^{s} | |_{F}^{2}$) or minimizing the KL-divergence (i.e., $D_{K L} ( P_{\mathcal{S}} || P_{\mathcal{T}} )$). However, both approaches rely on the estimation of the covariance matrix $\Sigma$, whose error is significantly amplified in high-dimensional data. Therefore, we propose to retain only the diagonal elements of $\Sigma$, which reduces the estimation error by a factor of $d$, as supported by the following theorem:
\begin{theorem}
\label{thm:covariance}
Given \( x_1, \dots, x_n \in \mathbb{R}^d \) independently drawn from a multivariate normal distribution \( \mathcal{N}(\mu, \Sigma) \), let \( \hat{\Sigma} \) be the unbiased sample covariance matrix and \( \hat{\sigma}^2 = [\hat{\sigma}_1^2, \dots, \hat{\sigma}_d^2]^T \) be the vector of its diagonal entries. Then, the mean squared errors satisfy:
\begin{equation}
    \mathbb{E}\left[\|\hat{\Sigma} - \Sigma\|_F^2\right] = \mathcal{O}\left(\frac{d^2}{n}\right), 
    \mathbb{E}\left[\|\hat{\sigma}^2 - \sigma^2\|_2^2\right] = \mathcal{O}\left(\frac{d}{n}\right).
\end{equation}
\end{theorem}

Due to space limitations, the corresponding proof can be found in Appendix. Emerging research~\cite{NEURIPS2023_eef6aecf} has shown that prompt tuning can serve as universal approximators for sequence-to-sequence functions. Motivated by this, we employ prompts as an implicit mapping from the target feature space to the source feature space. 
For the data $\mathbf{x}^u$ and the $i$-th layer of the modality-specific encoder $\Phi^u$, the input sequence $\boldsymbol{E}_{i-1}^u(\mathbf{x}^u)$ undergoes attention interaction with the added prompts $\boldsymbol{P}_{i-1}^u$ to obtain the transformed output sequence $\boldsymbol{E}_{i}^u(\mathbf{x}^u)$. The global feature representation can be expressed as $\boldsymbol{Z}_{i}^u(\mathbf{x}^u) = \operatorname{MeanPool}(\boldsymbol{E}_{i}^u(\mathbf{x}^u))$. Subsequently, we minimize the following empirical risk on the current batch $\{(\mathbf{x}_j^{a}, \mathbf{x}_j^{v}) \}_{j=1}^{B}$:
\begin{equation}
\begin{aligned}
&{\cal L}_{\mathrm{PMGFA}} = \sum_{u\in\{a,v\}} \operatorname
{Disc}\left({P}_{\mathcal{S}, u}, P_{\mathcal{T}, u}\right) \\
&= \sum_{u\in\{a,v\}} \frac{1} {N} \sum_{i=1}^{N} 
\left( \left\| \hat{\boldsymbol{\mu}}_{i}^{t,u}-\hat{\boldsymbol{\mu}}_{i}^{s, u} \right\|_{2}+\left\| \hat{\boldsymbol{\sigma}}_{i}^{t,u}-\hat{\boldsymbol{\sigma}}_{i}^{s, u} \right\|_{2} \right),
\end{aligned}
\end{equation}
where $\operatorname{Disc}\left(\cdot, \cdot\right)$ denotes the mean of the layer-wise distribution discrepancy. For convenience, we will interchangeably use $\operatorname{Disc}^u$ and $\operatorname
{Disc}({P}_{\mathcal{S}, u}, P_{\mathcal{T}, u})$ in the following context. $\left\| \cdot \right\|_{2}$ denotes the Euclidean norm. $\hat{\boldsymbol{\mu}}_{i}^{t,u} = \sum_{j=1}^B \boldsymbol{Z}_{i}^u(\mathbf{x}_j^{u}) / B$ and $\hat{\boldsymbol{\sigma}}_{i}^{t,u} = \sqrt{\sum_{j=1}^B [(\boldsymbol{Z}_{i}^u(\mathbf{x}_j^{u}) - \hat{\boldsymbol{\mu}}_{i}^{t,u})^2] / (B-1)}$ are the estimated mean and standard deviation, respectively. 
Similar to many other TTA methods~\cite{EATA, RMT, DA-TTA}, we pre-compute $\{\hat{\boldsymbol{\mu}}_{i}^{s, u}, \hat{\boldsymbol{\sigma}}_{i}^{s, u}\}_{i=1}^N$ offline prior to the test phase, and this process is performed only once.

\begin{table*}[t]
\centering
\setlength{\tabcolsep}{0.99mm}
\small
\begin{tabular}{lcccccccccccccccc}
\toprule
 &
  \multicolumn{3}{c}{Noise} &
  \multicolumn{4}{c}{Blur} &
  \multicolumn{4}{c}{Weather} &
  \multicolumn{4}{c}{Digital} &
   \\ \cmidrule(lr){2-4}\cmidrule(lr){5-8}\cmidrule(lr){9-12}\cmidrule(lr){13-16}
Method &
  Gauss. &
  Shot &
  Impul. &
  Defoc. &
  Glass &
  Motion &
  Zoom &
  Snow &
  Frost &
  Fog &
  Bright. &
  Contr. &
  Elast. &
  Pixel. &
  Jpeg &
  \cellcolor[cmyk]{0.1012,0.0945,0,0}Avg. \\ \midrule
Source &
  48.2 &
  50.0 &
  49.2 &
  67.7 &
  61.6 &
  70.6 &
  66.1 &
  60.9 &
  60.7 &
  44.7 &
  75.9 &
  51.8 &
  65.5 &
  68.7 &
  66.1 &
  \cellcolor[cmyk]{0.1012,0.0945,0,0}60.5 \\
$\bullet$ Tent$_{ICLR2021}$ &
  48.2 &
  49.8 &
  48.7 &
  67.7 &
  62.1 &
  70.8 &
  67.2 &
  61.8 &
  61.4 &
  33.7 &
  76.0 &
  51.2 &
  66.6 &
  69.6 &
  66.9 &
  \cellcolor[cmyk]{0.1012,0.0945,0,0}60.1 \\
$\bullet$ EATA$_{ICML2022}$ &
  48.7 &
  50.4 &
  49.6 &
  67.8 &
  63.2 &
  70.8 &
  67.5 &
  62.5 &
  62.5 &
  47.9 &
  76.1 &
  52.2 &
  66.9 &
  69.7 &
  67.4 &
  \cellcolor[cmyk]{0.1012,0.0945,0,0}61.5 \\
$\bullet$ SAR$_{ICLR2023}$ &
  48.5 &
  50.2 &
  49.2 &
  67.8 &
  63.8 &
  70.9 &
  67.9 &
  63.1 &
  62.7 &
  38.7 &
  76.1 &
  52.2 &
  67.1 &
  69.8 &
  67.4 &
  \cellcolor[cmyk]{0.1012,0.0945,0,0}61.0 \\
$\bullet$ DeYO$_{ICLR2024}$ &
  48.6 &
  50.2 &
  49.4 &
  67.9 &
  62.6 &
  70.9 &
  67.4 &
  62.5 &
  62.3 &
  40.4 &
  76.1 &
  52.2 &
  66.8 &
  69.8 &
  67.3 &
  \cellcolor[cmyk]{0.1012,0.0945,0,0}61.0 \\
$\bullet$ FOA$_{ICML2024}$ &
  49.2 &
  50.8 &
  49.7 &
  66.0 &
  65.5 &
  69.8 &
  67.4 &
  62.8 &
  65.7 &
  \underline{60.3} &
  74.9 &
  51.9 &
  \underline{69.5} &
  68.8 &
  68.0 &
  \cellcolor[cmyk]{0.1012,0.0945,0,0}62.7 \\
$\bullet$ READ$^\dag$$_{ICLR2024}$ &
  50.7 &
  52.2 &
  51.4 &
  67.9 &
  65.3 &
  71.1 &
  68.7 &
  64.0 &
  65.8 &
  56.3 &
  76.3 &
  53.6 &
  68.7 &
  70.0 &
  68.6 &
  \cellcolor[cmyk]{0.1012,0.0945,0,0}63.4 \\
$\bullet$ ABPEM$^\dag$$_{AAAI2025}$ &
  \underline{52.1} &
  \underline{53.1} &
  \underline{52.8} &
  \textbf{69.0} &
  \underline{65.6} &
  \underline{71.8} &
  \underline{68.8} &
  64.1 &
  65.7 &
  57.9 &
  \underline{76.6} &
  54.3 &
  69.2 &
  71.1 &
  \underline{69.2} &
  \cellcolor[cmyk]{0.1012,0.0945,0,0}\underline{64.1} \\
$\bullet$ SuMi$^\dag$$_{ICLR2025}$ &
  50.1 &
  50.7 &
  50.4 &
  \underline{68.2} &
  \underline{65.6} &
  \textbf{72.2} &
  \textbf{69.7} &
  \underline{65.7} &
  \textbf{67.0} &
  56.5 &
  \textbf{77.1} &
  \underline{55.2} &
  69.3 &
  \underline{71.2} &
  68.9 &
  \cellcolor[cmyk]{0.1012,0.0945,0,0}63.9 \\
\rowcolor[cmyk]{0.193,0,0.2222,0} 
$\bullet$ BriMPR$^\dag$ &
  \textbf{55.3} &
  \textbf{56.1} &
  \textbf{56.7} &
  67.8 &
  \textbf{67.9} &
  70.6 &
  \underline{68.8} &
  \textbf{65.9} &
  \underline{66.2} &
  \textbf{64.1} &
  76.2 &
  \textbf{56.3} &
  \textbf{72.0} &
  \textbf{73.7} &
  \textbf{70.5} &
  \textbf{65.9} \\ \midrule

Source &
  52.9 &
  53.0 &
  53.1 &
  57.2 &
  57.2 &
  58.5 &
  57.5 &
  56.5 &
  57.1 &
  55.6 &
  59.2 &
  53.7 &
  57.1 &
  56.4 &
  57.3 &
  \cellcolor[cmyk]{0.1012,0.0945,0,0}56.2 \\
$\bullet$ Tent$_{ICLR2021}$ &
  53.2 &
  53.3 &
  53.3 &
  56.8 &
  56.6 &
  57.9 &
  57.2 &
  55.9 &
  56.6 &
  56.5 &
  58.5 &
  53.9 &
  57.5 &
  56.8 &
  56.9 &
  \cellcolor[cmyk]{0.1012,0.0945,0,0}56.1 \\
$\bullet$ EATA$_{ICML2022}$ &
  53.4 &
  53.5 &
  53.5 &
  57.0 &
  57.0 &
  58.3 &
  57.7 &
  56.3 &
  57.0 &
  56.8 &
  59.1 &
  54.2 &
  57.9 &
  57.2 &
  57.2 &
  \cellcolor[cmyk]{0.1012,0.0945,0,0}56.4 \\
$\bullet$ SAR$_{ICLR2023}$ &
  53.3 &
  53.3 &
  53.3 &
  56.4 &
  56.5 &
  57.9 &
  57.3 &
  55.6 &
  56.4 &
  56.3 &
  58.8 &
  53.7 &
  57.8 &
  56.9 &
  57.0 &
  \cellcolor[cmyk]{0.1012,0.0945,0,0}56.0 \\
$\bullet$ DeYO$_{ICLR2024}$ &
  53.3 &
  53.4 &
  53.4 &
  56.7 &
  56.7 &
  58.0 &
  57.3 &
  56.0 &
  56.8 &
  56.4 &
  58.7 &
  53.9 &
  57.7 &
  57.0 &
  57.0 &
  \cellcolor[cmyk]{0.1012,0.0945,0,0}56.2 \\
$\bullet$ FOA$_{ICML2024}$ &
  52.7 &
  52.7 &
  52.7 &
  53.2 &
  53.6 &
  53.6 &
  53.8 &
  53.4 &
  53.4 &
  53.3 &
  55.6 &
  52.5 &
  55.3 &
  53.7 &
  54.4 &
  \cellcolor[cmyk]{0.1012,0.0945,0,0}53.6 \\
$\bullet$ READ$^\dag$$_{ICLR2024}$ &
  53.8 &
  54.0 &
  \underline{53.8} &
  \underline{58.0} &
  57.9 &
  \underline{59.2} &
  \underline{58.7} &
  57.1 &
  \textbf{58.2} &
  50.0 &
  \underline{60.0} &
  \textbf{55.2} &
  58.5 &
  \underline{57.7} &
  \underline{58.2} &
  \cellcolor[cmyk]{0.1012,0.0945,0,0}56.7 \\
$\bullet$ ABPEM$^\dag$$_{AAAI2025}$ &
  46.5 &
  46.7 &
  46.5 &
  54.2 &
  55.1 &
  56.4 &
  55.2 &
  51.3 &
  53.2 &
  52.1 &
  56.6 &
  52.1 &
  54.4 &
  51.7 &
  54.7 &
  \cellcolor[cmyk]{0.1012,0.0945,0,0}52.4 \\
$\bullet$ SuMi$^\dag$$_{ICLR2025}$ &
  \underline{54.0} &
  \underline{54.3} &
  \underline{53.8} &
  \textbf{58.2} &
  \underline{58.4} &
  \textbf{59.4} &
  \underline{58.7} &
  \underline{57.5} &
  \textbf{58.2} &
  \underline{57.6} &
  59.4 &
  \underline{54.8} &
  \underline{59.0} &
  57.5 &
  \underline{58.2} &
  \cellcolor[cmyk]{0.1012,0.0945,0,0}\underline{57.3} \\
\rowcolor[cmyk]{0.193,0,0.2222,0} 
$\bullet$ BriMPR$^\dag$ &
  \textbf{54.9} &
  \textbf{55.0} &
  \textbf{55.0} &
  57.9 &
  \textbf{58.5} &
  58.9 &
  \textbf{58.7} &
  \textbf{57.5} &
  \underline{58.0} &
  \textbf{58.5} &
  \textbf{60.3} &
  54.5 &
  \textbf{59.7} &
  \textbf{59.3} &
  \textbf{59.0} &
  \textbf{57.7} \\ \bottomrule
\end{tabular}
\caption{Comparison with SOTA methods on Kinetics50-C (top) and VGGSound-C (bottom) under the unimodal shift setting (severity level 5 of video corruption). $^\dag$Multimodal test-time adaptation methods.}
\label{tab:video}
\end{table*}

\begin{table*}[t]
\centering
\setlength{\tabcolsep}{1mm}
\small
\begin{tabular}{lcccccccccccccc}
\toprule
 &
  \multicolumn{3}{c}{Noise} &
  \multicolumn{3}{c}{Weather} &
   &
  \multicolumn{3}{c}{Noise} &
  \multicolumn{3}{c}{Weather} &
   \\ \cmidrule(lr){2-4} \cmidrule(lr){5-7} \cmidrule(lr){9-11} \cmidrule(lr){12-14}
Method &
  Gauss. &
  Traff. &
  Crowd &
  Rain &
  Thund. &
  Wind &
  \cellcolor[cmyk]{0.1012,0.0945,0,0}Avg. &
  Gauss. &
  Traff. &
  Crowd &
  Rain &
  Thund. &
  Wind &
  \cellcolor[cmyk]{0.1012,0.0945,0,0}Avg. \\ \midrule
Source &
  74.3 &
  65.3 &
  68.0 &
  70.3 &
  68.0 &
  70.5 &
  \cellcolor[cmyk]{0.1012,0.0945,0,0}69.4 &
  37.3 &
  21.2 &
  16.9 &
  21.8 &
  27.3 &
  25.7 &
  \cellcolor[cmyk]{0.1012,0.0945,0,0}25.0 \\
$\bullet$ Tent$_{ICLR2021}$ &
  74.6 &
  67.4 &
  69.5 &
  70.8 &
  67.6 &
  71.2 &
  \cellcolor[cmyk]{0.1012,0.0945,0,0}70.2 &
  10.8 &
  2.8 &
  1.8 &
  2.9 &
  5.6 &
  3.9 &
  \cellcolor[cmyk]{0.1012,0.0945,0,0}4.6 \\
$\bullet$ EATA$_{ICML2022}$ &
  74.6 &
  67.3 &
  69.4 &
  70.8 &
  69.8 &
  71.0 &
  \cellcolor[cmyk]{0.1012,0.0945,0,0}70.5 &
  \underline{40.2} &
  \underline{30.0} &
  27.8 &
  29.7 &
  36.5 &
  32.2 &
  \cellcolor[cmyk]{0.1012,0.0945,0,0}32.7 \\
$\bullet$ SAR$_{ICLR2023}$ &
  74.6 &
  67.0 &
  69.2 &
  70.9 &
  69.5 &
  70.9 &
  \cellcolor[cmyk]{0.1012,0.0945,0,0}70.3 &
  30.4 &
  5.5 &
  8.0 &
  9.3 &
  32.5 &
  17.2 &
  \cellcolor[cmyk]{0.1012,0.0945,0,0}17.1 \\
$\bullet$ DeYO$_{ICLR2024}$ &
  74.6 &
  67.0 &
  69.3 &
  70.8 &
  69.0 &
  71.0 &
  \cellcolor[cmyk]{0.1012,0.0945,0,0}70.3 &
  22.9 &
  4.9 &
  15.8 &
  4.9 &
  16.5 &
  20.0 &
  \cellcolor[cmyk]{0.1012,0.0945,0,0}14.2 \\
$\bullet$ FOA$_{ICML2024}$ &
  73.8 &
  \textbf{70.0} &
  70.5 &
  71.0 &
  \textbf{73.0} &
  71.2 &
  \cellcolor[cmyk]{0.1012,0.0945,0,0}71.6 &
  31.5 &
  26.2 &
  23.7 &
  31.0 &
  34.2 &
  26.7 &
  \cellcolor[cmyk]{0.1012,0.0945,0,0}28.9 \\
$\bullet$ READ$^\dag$$_{ICLR2024}$ &
  \underline{74.8} &
  69.2 &
  69.9 &
  71.4 &
  72.4 &
  71.0 &
  \cellcolor[cmyk]{0.1012,0.0945,0,0}71.5 &
  39.9 &
  29.4 &
  26.8 &
  30.8 &
  36.8 &
  30.7 &
  \cellcolor[cmyk]{0.1012,0.0945,0,0}32.4 \\
$\bullet$ ABPEM$^\dag$$_{AAAI2025}$ &
  74.7 &
  68.5 &
  70.3 &
  \textbf{71.7} &
  72.3 &
  71.2 &
  \cellcolor[cmyk]{0.1012,0.0945,0,0}71.4 &
  38.5 &
  27.6 &
  25.2 &
  26.5 &
  32.7 &
  26.5 &
  \cellcolor[cmyk]{0.1012,0.0945,0,0}29.5 \\
$\bullet$ SuMi$^\dag$$_{ICLR2025}$ &
  \textbf{75.1} &
  68.9 &
  \underline{70.6} &
  \underline{71.6} &
  \underline{72.8} &
  \textbf{72.1} &
  \cellcolor[cmyk]{0.1012,0.0945,0,0}\underline{71.9} &
  \textbf{41.9} &
  26.3 &
  \underline{27.9} &
  \underline{31.6} &
  \underline{37.1} &
  \underline{34.1} &
  \cellcolor[cmyk]{0.1012,0.0945,0,0}\underline{33.2} \\
\rowcolor[cmyk]{0.193,0,0.2222,0} 
$\bullet$ BriMPR$^\dag$ &
  \underline{74.8} &
  \underline{69.6} &
  \textbf{71.7} &
  71.5 &
  72.4 &
  \underline{72.0} &
  \textbf{72.0} &
  39.3 &
  \textbf{35.0} &
  \textbf{36.7} &
  \textbf{32.5} &
  \textbf{41.0} &
  \textbf{34.6} &
  \textbf{36.5} \\ \bottomrule
\end{tabular}
\caption{Comparison with SOTA methods on Kinetics50-C (left) and VGGSound-C (right) under the unimodal shift setting (severity level 5 of audio corruption).}
\label{tab:audio}
\end{table*}

\subsection{Inter-modal Interaction Enhancement for Alignment Refinement}
\label{IIEAR}
After initial cross-modal semantic alignment via unimodal feature calibration, we further improve the quality of alignment by inter-modal interactions. By recombining masked and complete modalities, the unmasked low-quality modality is forced to draw multimodal information from credible pseudo-labels. Meanwhile, inter-modal instance-wise contrastive learning is applied to strengthen the alignment across instances.

\subsubsection{Cross-modal Masked Embedding Recombination.}
\label{CMER}
Masked language modeling~\cite{BERT} and masked image modeling~\cite{MAE} force model to reconstruct the masked regions by utilizing contextual clues and have been widely used as powerful self-supervised learning paradigms in natural language processing and computer vision tasks, respectively. 
Related but distinct, our proposed Cross-modal Masked Embedding Recombination (CMER) uses masking to simulate distribution shifts from missing patches, serving as a form of data augmentation.

For input $\mathbf{x}^u$, we randomly mask a portion (e.g., 50\%) of its patches and encode the unmasked part $\mathbf{x}^{u_m}$ using $\Phi^u$ with modality-specific prompts $\boldsymbol{P}^u$ to obtain the masked embedding $\Phi^u(\mathbf{x}^{u_m})$. Then, $\Phi^u(\mathbf{x}^{u_m})$ is recombined with complete embeddings from other modalities and passed to the joint module, generating an augmented representation that simulates unimodal corruption. Taking the masked audio modality as an example, the recombined representations and their predictions are formulated as:
\begin{equation}
\label{augpred}
\begin{aligned}
\boldsymbol{Z}^{a_m v} &= \Psi([\Phi^a(\mathbf{x}^{a_m});\Phi^v(\mathbf{x}^{v})]), \\
\boldsymbol{y}^{a_m v} &= \sigma(h(\text{MealPool}(\boldsymbol{Z}^{a_m v}))),
\end{aligned}
\end{equation}
where $\sigma$ denotes the softmax function. With the initial alignment from PMGFA, we can utilize the complete multimodal data to provide reliable pseudo-labels for augmented inputs. 
As pseudo-labels become more reliable in the later stages of adaptation, we further calibrate them via temperature scaling~\cite{hinton2015distillingknowledgeneuralnetwork, temperature-scaling}:
\begin{equation}
\label{pseudolabel}
\hat{\boldsymbol{y}}^{a v}_k = \frac{\exp([h(\text{MealPool}(\boldsymbol{Z}^{a v}))]_k / \text{AdaTp})}{\sum_{k^\prime=1}^C \exp([h(\text{MealPool}(\boldsymbol{Z}^{a v}))]_{k^\prime} / \text{AdaTp})}.
\end{equation}
Here, $k$ and $k^{\prime}$ denote the $k$-th and $k^{\prime}$-th elements of the tensor, and $C$ represents the number of classes. $\text{AdaTp} = 1 + \tau_{0} / (1 + \exp(D_0 - \operatorname{Disc}^{J})) \in (1, 1+\tau_{0})$ is the adaptive temperature coefficient, where $\operatorname{Disc}^{J}$ is the distribution discrepancy calculated for the joint module, and $\tau_{0}$ and $D_0$ are predefined hyperparameters. 
When $\operatorname{Disc}^{J}$ is large, $\text{AdaTp}$ approaches $1+\tau_{0}$ to alleviate overconfident predictions. As $\operatorname{Disc}^{J}$ decreases, $\text{AdaTp}$ approaches 1, and Eq.~\eqref{pseudolabel} approximates the vanilla softmax function. Subsequently, minimize the cross-entropy between the calibrated pseudo-label and the augmented predictions:
\begin{equation}
\begin{aligned}
&\mathcal{L}_\text{CMER} = \lambda^{a} \mathcal{L}_{a_mv} + \lambda^{v} \mathcal{L}_{av_m} \\
&= - \lambda^{a} \sum_{k=1}^C \hat{\boldsymbol{y}}^{av}_k\log \boldsymbol{y}^{a_mv}_k - \lambda^{v} \sum_{k=1}^C \hat{\boldsymbol{y}}^{av}_k\log \boldsymbol{y}^{av_m}_k,
\end{aligned}
\end{equation}
where $\lambda^u = 1 - \text{Disc}^u / (\text{Disc}^a + \text{Disc}^v)$ $(u \in\{a, v\})$ is the weight of the corresponding term, assigning a higher weight to the augmentation with a milder distribution shift in the masked modality. Intuitively, $\mathcal{L}_\text{CMER}$ deliberately discards high-quality modality information, forcing the corrupted modality to independently derive the correct result.

\subsubsection{Inter-modal Instance-wise Contrastive Learning.}
\label{IICL}
Contrastive learning~\cite{MoCo, SimCLR} has emerged as a key paradigm in cross-modal representation learning, aiming to improve the quality of representations by aligning the feature spaces of the same semantic instance across different modalities/views. Building upon the calibration of unimodal feature distributions, BriMPR introduces inter-modal instance-wise contrastive learning. For data $\mathbf{x}^u$ ($u\in\{a, v\}$), its unimodal representation is as follows: 
\begin{equation}
\label{unifeature}
\boldsymbol{Z}^{u} = \Psi(\Phi^u(\mathbf{x}^{u})).
\end{equation}
Subsequently, different unimodal representations of the same instance are regarded as positive pairs, while the others as negative pairs. The contrastive loss is defined as:
\begin{equation}
\label{iicl}
{\cal L}_{\mathrm{IICL}}=-\frac{1} {2B} \sum_{j=1}^{B}\sum_{u_1 \ne u_2}
\log \frac{e^{\operatorname{sim} ( \boldsymbol{Z}^{u_1}_{j}, \boldsymbol{Z}^{u_2}_{j} ) /\tau}} 
{\sum_{j^\prime = 1}^B e^{\operatorname{sim} ( \boldsymbol{Z}^{u_1}_{j}, \boldsymbol{Z}^{u_2}_{j^\prime} ) /\tau}},
\end{equation}
where $\operatorname{sim}(\cdot, \cdot)$ denotes the cosine similarity function, and $\tau$ denotes the temperature hyperparameter.

\subsection{Overall Procedure}
\label{Overall Procedure}
To brief, BriMPR optimizes the added modality-specific prompts by minimizing the following loss:
\begin{equation}
\label{overallloss}
\cal{L}_{\text{BriMPR}} = \cal{L}_{\text{PMGFA}} + \cal{L}_{\text{CMER}} + \cal{L}_{\text{IICL}}.
\end{equation}

\begin{table*}[t]
\centering
\setlength{\tabcolsep}{0.99mm}
\small
\begin{tabular}{lcccccccccccccccc}
\toprule
 &
  \multicolumn{3}{c}{Noise} &
  \multicolumn{4}{c}{Blur} &
  \multicolumn{4}{c}{Weather} &
  \multicolumn{4}{c}{Digital} &
   \\ \cmidrule(lr){2-4}\cmidrule(lr){5-8}\cmidrule(lr){9-12}\cmidrule(lr){13-16}
Method &
  Gauss. &
  Shot &
  Impul. &
  Defoc. &
  Glass &
  Motion &
  Zoom &
  Snow &
  Frost &
  Fog &
  Bright. &
  Contr. &
  Elast. &
  Pixel. &
  Jpeg &
  \cellcolor[cmyk]{0.1012,0.0945,0,0}Avg. \\ \midrule
Source &
  13.1 &
  14.1 &
  13.3 &
  37.2 &
  37.4 &
  45.3 &
  41.8 &
  29.4 &
  32.6 &
  20.4 &
  55.2 &
  18.3 &
  42.5 &
  38.8 &
  37.8 &
  \cellcolor[cmyk]{0.1012,0.0945,0,0}31.8 \\
$\bullet$ Tent$_{ICLR2021}$ &
  9.1 &
  9.7 &
  9.1 &
  32.5 &
  34.1 &
  43.5 &
  40.2 &
  23.2 &
  28.3 &
  13.2 &
  55.1 &
  13.7 &
  40.7 &
  34.7 &
  35.0 &
  \cellcolor[cmyk]{0.1012,0.0945,0,0}28.1 \\
$\bullet$ EATA$_{ICML2022}$ &
  12.9 &
  14.0 &
  13.1 &
  38.1 &
  38.7 &
  46.9 &
  43.1 &
  30.6 &
  33.0 &
  20.2 &
  56.5 &
  18.2 &
  43.7 &
  40.7 &
  39.0 &
  \cellcolor[cmyk]{0.1012,0.0945,0,0}32.6 \\
$\bullet$ SAR$_{ICLR2023}$ &
  11.8 &
  12.8 &
  11.9 &
  37.4 &
  38.2 &
  46.3 &
  43.1 &
  29.8 &
  33.0 &
  17.4 &
  56.0 &
  16.0 &
  43.5 &
  39.5 &
  38.2 &
  \cellcolor[cmyk]{0.1012,0.0945,0,0}31.7 \\
$\bullet$ DeYO$_{ICLR2024}$ &
  11.0 &
  12.0 &
  11.1 &
  37.0 &
  37.7 &
  46.3 &
  43.2 &
  29.9 &
  33.3 &
  17.9 &
  56.2 &
  17.0 &
  43.7 &
  39.7 &
  38.0 &
  \cellcolor[cmyk]{0.1012,0.0945,0,0}31.6 \\
$\bullet$ FOA$_{ICML2024}$ &
  18.7 &
  20.6 &
  19.3 &
  43.7 &
  \textbf{45.5} &
  50.2 &
  \underline{47.9} &
  \textbf{38.9} &
  \textbf{43.7} &
  \textbf{37.2} &
  \textbf{60.5} &
  23.5 &
  \underline{52.7} &
  \underline{48.9} &
  \underline{47.4} &
  \cellcolor[cmyk]{0.1012,0.0945,0,0}\underline{39.9} \\
$\bullet$ READ$^\dag$$_{ICLR2024}$ &
  14.5 &
  14.9 &
  14.8 &
  \underline{43.8} &
  42.1 &
  \underline{51.0} &
  46.5 &
  35.4 &
  38.9 &
  27.6 &
  58.9 &
  22.6 &
  47.1 &
  42.1 &
  38.1 &
  \cellcolor[cmyk]{0.1012,0.0945,0,0}35.9 \\
$\bullet$ ABPEM$^\dag$$_{AAAI2025}$ &
  \underline{19.2} &
  \underline{20.7} &
  \underline{19.7} &
  \textbf{46.2} &
  44.2 &
  \textbf{51.9} &
  \underline{47.9} &
  \underline{38.1} &
  \underline{41.1} &
  32.6 &
  \underline{59.9} &
  \underline{25.3} &
  49.4 & 
  48.8 &
  45.6 &
  \cellcolor[cmyk]{0.1012,0.0945,0,0}39.4 \\
$\bullet$ SuMi$^\dag$$_{ICLR2025}$ &
  12.5 &
  13.6 &
  12.6 &
  37.0 &
  37.9 &
  45.9 &
  42.3 &
  29.3 &
  32.7 &
  19.7 &
  55.7 &
  17.8 &
  42.7 &
  38.3 &
  36.9 &
  \cellcolor[cmyk]{0.1012,0.0945,0,0}31.7 \\
\rowcolor[cmyk]{0.193,0,0.2222,0} 
$\bullet$ BriMPR$^\dag$ &
  \textbf{22.9} &
  \textbf{24.2} &
  \textbf{24.1} &
  43.6 &
  \underline{45.4} &
  49.5 &
  \textbf{48.2} &
  38.0 &
  40.8 &
  \underline{36.8} &
  59.8 &
  \textbf{27.1} &
  \textbf{52.8} &
  \textbf{52.7} &
  \textbf{47.9} &
  \textbf{40.9} \\ \bottomrule
\end{tabular}
\caption{Comparison with SOTA methods on Kinetics50-C under the multimodal shift setting (severity level 5).}
\label{tab:ks50_both}
\end{table*}
\begin{table}[t]
\centering
\setlength{\tabcolsep}{1mm}
\small
\begin{tabular}{lccccccc}
\toprule
& \multicolumn{3}{c}{Noise} & \multicolumn{3}{c}{Weather} & \\ \cmidrule(lr){2-4} \cmidrule(lr){5-7}

Method & Gauss. & Traff. & Crowd & Rain & Thund. & Wind & \cellcolor[cmyk]{0.1012,0.0945,0,0}Avg. \\ \midrule

Source & 17.1 & 6.4 & 5.4 & 6.0 & 13.5 & 8.8 & \cellcolor[cmyk]{0.1012,0.0945,0,0}9.5 \\
$\bullet$ Tent & 3.2 & 0.9 & 0.8 & 0.9 & 2.8 & 1.3 & \cellcolor[cmyk]{0.1012,0.0945,0,0}1.6 \\
$\bullet$ EATA & 21.5 & 7.7 & 7.1 & 7.3 & 17.3 & 11.9 & \cellcolor[cmyk]{0.1012,0.0945,0,0}12.1 \\
$\bullet$ SAR & 10.7 & 1.8 & 1.6 & 2.3 & 12.8 & 3.1 & \cellcolor[cmyk]{0.1012,0.0945,0,0}5.4 \\
$\bullet$ DeYO & 6.7 & 1.2 & 1.3 & 1.3 & 9.3 & 2.9 & \cellcolor[cmyk]{0.1012,0.0945,0,0}3.8 \\
$\bullet$ FOA & 18.8 & 10.8 & 11.4 & \underline{11.6} & 20.5 & 10.4 & \cellcolor[cmyk]{0.1012,0.0945,0,0}13.9 \\
$\bullet$ READ$^\dag$ & 20.1 & 12.5 & 10.7 & 10.5 & \underline{20.5} & \underline{13.4} & \cellcolor[cmyk]{0.1012,0.0945,0,0}14.6 \\
$\bullet$ ABPEM$^\dag$ & \underline{21.9} & \underline{13.4} & \underline{12.3} & 10.9 & 20.4 & 12.4 & \cellcolor[cmyk]{0.1012,0.0945,0,0}\underline{15.2} \\
$\bullet$ SuMi$^\dag$ & 17.0 & 6.8 & 5.7 & 6.2 & 13.4 & 8.8 & \cellcolor[cmyk]{0.1012,0.0945,0,0}9.7 \\

\rowcolor[cmyk]{0.193,0,0.2222,0} 
$\bullet$ BriMPR$^\dag$ & \textbf{23.5} & \textbf{18.8} & \textbf{21.4} & \textbf{15.8} & \textbf{26.8} & \textbf{18.3} & \textbf{20.7} \\ \bottomrule
\end{tabular}
\caption{Comparison with SOTA methods on VGGSound-C under the multimodal shift setting (severity level 5).}
\label{tab:vgg_both}
\end{table}
\begin{table}[t]
\centering
\small
\begin{tabular}{lcccc}
\toprule
 & \multicolumn{2}{c}{MOSI $\to$ SIMS} & \multicolumn{2}{c}{SIMS $\to$ MOSI} \\ \cmidrule(lr){2-3} \cmidrule(lr){4-5}
Method & ACC$\uparrow$ & F1$\uparrow$ & ACC$\uparrow$ & F1$\uparrow$ \\ \midrule
Source & 46.0 & 45.6 & 59.0 & 73.6 \\
$\bullet$ Tent & 38.1 & 42.2 & 59.6 & 74.5 \\
$\bullet$ READ$^\dag$ & 32.4 & 44.5 & 59.7 & 74.7 \\
$\bullet$ SuMi$^\dag$ & 44.4 & 45.0 & 59.4 & 74.2 \\
\rowcolor[cmyk]{0.193,0,0.2222,0} 
$\bullet$ BriMPR$^\dag$ & \textbf{58.2} & \textbf{57.6} & \textbf{59.9} & \textbf{74.9} \\ \bottomrule
\end{tabular}
\caption{Comparison with SOTA methods on real-world shift datasets.}
\label{tab:mosi_sims}
\end{table}

\section{Experiments}
\label{Experiments}

\subsection{Experimental Setups}
\label{experimental_setups}
\subsubsection{Datasets and models.}
We evaluate our method on four commonly used multimodal datasets, including Kinetics50-C, VGGSound-C~\cite{READ}, CMU-MOSI~\cite{MOSI}, and CH-SIMS~\cite{SIMS}. 
Kinetics50-C/VGGSound-C contain two modalities: video and audio, and are obtained by adding various corruptions to the test sets of the original versions (i.e., Kinetics~\cite{ks} and VGGSound~\cite{vgg}). For the video modality and the audio modality, 15 and 6 types of corruption are introduced, respectively, which are divided into 5 severity levels. Following~\cite{READ}, we use the pre-trained CAV-MAE~\cite{cavmae} as the source model.
CMU-MOSI/CH-SIMS contain three modalities: text, video, and audio. Following~\cite{SuMi}, we use stacked Transformer blocks as the backbone and pre-train the model on MOSI and SIMS, respectively.

\subsubsection{Considered settings.}
For domain shifts caused by corruptions, we consider two tasks and report average classification accuracy (\%): (1) Under the unimodal shift setting, following~\cite{READ}, one modality is corrupted while the other modality remains clean; (2) Under the multimodal shift setting, both modalities are corrupted. For real-world domain shifts, we consider the settings of MOSI → SIMS and SIMS → MOSI, and report accuracy (ACC) and F1 score (F1).

\subsubsection{Baselines.}
We compare the proposed method with multiple baselines including Source (source pre-trained model), Tent~\cite{Tent}, EATA~\cite{EATA}, SAR~\cite{SAR}, DeYO~\cite{DeYO}, FOA~\cite{FOA}, READ~\cite{READ}, ABPEM~\cite{ABPEM} and SuMi~\cite{SuMi}.

\subsubsection{Implementation details.}
For all experiments, we use an Adam optimizer with a learning rate of 1e-4 and batch size of 64. The default number of prompts per layer $m_p$ is set to 10 and the prompts are randomly initialized~\cite{VPT}. The mask ratio is set to 0.5. $\tau_0$ and $D_0$ of the adaptive temperature coefficient $\text{AdaTp}$ are set to 0.2 and 5 respectively. $\tau$ in Eq.~\eqref{iicl} is set to 0.07/0.25 for the unimodal and multimodal corruption settings respectively. 
For the hyperparameters of the compared methods, we adopt the recommended values from the respective papers. All the experiments are conducted with 3 random seeds on RTX-3090 GPUs.

\begin{table*}[t]
\centering
\setlength{\tabcolsep}{1mm}
\small
\begin{tabular}{lcccccc}
\toprule
& \multicolumn{3}{c}{Kinetics50-C} & \multicolumn{3}{c}{VGGSound-C} \\ \cmidrule(l){2-4} \cmidrule(l){5-7}
\multicolumn{1}{c}{Method} & audio & video & both & audio & video & both \\ \midrule
BriMPR w/o ${\cal L}_{\mathrm{CMER}}$ & 71.4 & 65.6 & 40.7 & 35.3 & 57.6 & 20.2 \\
$\bullet$ BriMPR ($\lambda^a \leftrightarrow \lambda^v$) & 70.0 \small{{\color[cmyk]{0,0.9444,0.9817,0} (-1.4)}} & 65.2 \small{{\color[cmyk]{0,0.9444,0.9817,0} (-0.4)}} & 39.9 \small{{\color[cmyk]{0,0.9444,0.9817,0} (-0.8)}} & 32.1 \small{{\color[cmyk]{0,0.9444,0.9817,0} (-3.2)}} & 56.5 \small{{\color[cmyk]{0,0.9444,0.9817,0} (-1.1)}} & 19.5 \small{{\color[cmyk]{0,0.9444,0.9817,0} (-0.7)}} \\
$\bullet$ BriMPR & 72.0 \small{{\color[cmyk]{0.6323,0,0.964,0} (+0.6)}} & 65.9 \small{{\color[cmyk]{0.6323,0,0.964,0} (+0.3)}} & 40.9 \small{{\color[cmyk]{0.6323,0,0.964,0} (+0.2)}} & 36.5 \small{{\color[cmyk]{0.6323,0,0.964,0} (+1.2)}} & 57.7 \small{{\color[cmyk]{0.6323,0,0.964,0} (+0.1)}} & 20.7 \small{{\color[cmyk]{0.6323,0,0.964,0} (+0.5)}} \\ \bottomrule
\end{tabular}
\caption{Verify the effect of CMER from the perspective of weights.}
\label{tab:lambda_CMER}
\end{table*}
\begin{table}[t]
\centering
\setlength{\tabcolsep}{1mm}
\small
\begin{tabular}{lcccccc}
\toprule
& \multicolumn{3}{c}{Kinetics50-C} & \multicolumn{3}{c}{VGGSound-C} \\ \cmidrule(l){2-4} \cmidrule(l){5-7}
\multicolumn{1}{c}{Method} & audio & video & both & audio & video & both \\ \midrule

Source & 69.4 & 60.5 & 31.8 & 25.0 & 56.2 & 9.5 \\
${\cal L}_\text{KL}$ & 69.3 & 60.4 & 31.5 & 24.8 & 55.7 & 9.1 \\
${\cal L}_{\text{moment}_2}$ & 69.9 & 61.5 & 34.5 & 25.2 & 48.9 & 12.1 \\
${\cal L}_{\text{moment}_1}$ & 71.3 & 63.5 & 37.4 & 32.0 & 54.7 & 16.4 \\ \midrule
($\textbf{A}$) ${\cal L}_{\mathrm{PMGFA}}$ & 71.1 & 64.7 & 40.5 & 35.1 & 57.5 & 20.1 \\
($\textbf{B}$) + ${\cal L}_{\mathrm{IICL}}$ & 71.4 & 65.6 & 40.7 & 35.3 & 57.6 & 20.2 \\
\rowcolor[cmyk]{0.193,0,0.2222,0} 
($\textbf{C}$) + ${\cal L}_{\mathrm{CMER}}$ & \textbf{72.0} & \textbf{65.9} & \textbf{40.9} & \textbf{36.5} & \textbf{57.7} & \textbf{20.7} \\ \bottomrule
\end{tabular}
\caption{Ablation studies for different components of BriMPR. ${\cal L}_\text{KL}$, ${\cal L}_{\text{moment}_2}$ and ${\cal L}_{\text{moment}_1}$ respectively denote replacing ${\cal L}_\text{PMFGA}$ with the KL-divergence, moment matching, and moment matching in a non-squared form.}
\label{tab:ablation}
\end{table}

\subsection{Performance Comparison}
\subsubsection{Results of the unimodal shift setting.}
In Tab.~\ref{tab:video} and Tab.~\ref{tab:audio}, we present the results of the unimodal shift setting on Kinetics50-C and VGGSound-C with audio corruption and video corruption, respectively. Our proposed method BriMPR consistently improves the source model and outperforms all other competing methods. 
Notably, in scenarios where the dominant modality of the dataset is corrupted (for Kinetics50-C, video is the dominant modality; for VGGSound-C, audio is the dominant modality), BriMPR yields significant performance gains (60.5\% $\to$ 65.9\% on Kinetics50-C; 25.0\% $\to$ 36.5\% on VGGSound-C). 

\subsubsection{Results of the multimodal shift setting.}
Tab.~\ref{tab:ks50_both} and Tab.~\ref{tab:vgg_both} respectively present the results of the challenging multimodal shift setting on Kinetics50-C and VGGSound-C. Taking the ``Gauss.'' column in Tab.~\ref{tab:ks50_both} as an example, the reported value denotes the average classification accuracy (\%) across all 6 types of audio corruption, given the presence of Gaussian corruption in the video modality. Most methods suffer significant performance drops under this setting, whereas our BriMPR achieves the best results on most domains by decoupling MMTTA into unimodal alignment sub-problems, thereby reducing the dependence on high-quality modalities.

\subsubsection{Results of the real-world shift setting.}
Tab.~\ref{tab:mosi_sims} presents the results from the MOSI/SIMS datasets using text, video, and audio modalities. BriMPR exhibits strong robustness to real-world shifts. Notably, only BriMPR achieves results better than random guess ($>50\%$) on the MOSI $\to$ SIMS task, thanks to its modulation of the target feature space.

\subsection{Ablation Studies}
\subsubsection{Scrutinize CMER from the perspective of the weight $\lambda^u$.}
To illustrate how multimodal test-time adaptation benefits from CMER, we swap the weights $\lambda^u$ ($u\in\{a,v\}$) in ${\cal L}_{\mathrm{CMER}}$, assigning lower weight to the augmentation with a milder distribution shift in the masked modality. As reported in Tab.~\ref{tab:lambda_CMER}, the mismatched weights lead to significant performance drops. Taking the case of audio corruption as an example (where $\lambda^{v} > \lambda^{a}$), the performance degradation can be attributed to two main factors: (1) For $\lambda^{a} \mathcal{L}_{av_m}$, the small $\lambda^{a}$ suppresses the extraction of multimodal information by the complete but low-quality audio modality; (2) For $\lambda^{v} \mathcal{L}_{a_mv}$, providing pseudo-labels to the augmentation with the masked audio modality introduces more error information into the unmasked high-quality video modality. 

\subsubsection{Component analysis.}
As shown in Tab.~\ref{tab:ablation}, we conducted an ablation study on the components of BriMPR. First, we verify the effectiveness of ${\cal L}_{\mathrm{PMGFA}}$ (\textbf{A}): compared with KL-divergence (Row 2) and moment matching (Row 3), ${\cal L}_{\mathrm{PMGFA}}$ demonstrates a significant advantage, as it eliminates the off-diagonal elements in the covariance matrix, reducing the estimation error. When moment matching is modified to a non-squared form (Row 4), performance improves in most cases, as the squared norm also amplifies the error. Subsequently, combining ${\cal L}_{\mathrm{PMGFA}}$ (\textbf{A}), which serves as the initial alignment objective, with inter-modal instance-wise contrastive learning ${\cal L}_{\mathrm{IICL}}$ (\textbf{B}) and cross-modal masked embedding recombination ${\cal L}_{\mathrm{CMER}}$ (\textbf{C}) for alignment refinement, leads to further performance gains across all tasks.

\section{Conclusion}
In this paper, we introduce BriMPR, a novel MMTTA method which tackles the coupling effect of unimodal feature shift and cross-modal semantic misalignment in a divide-and-conquer manner. Specifically, benefiting from the well-aligned source feature space, we first calibrate each unimodal global feature distribution via modality-specific prompts to achieve initial cross-modal semantic alignment. We then introduce a novel Cross-modal Masked Embedding Recombination strategy to facilitate the integration of multimodal information into low-quality modalities, and further refine the alignment via Inter-modal Instance-wise Contrastive Learning. Extensive experiments conducted on MMTTA benchmark, which includes corruption datasets and real-world shift datasets, demonstrate the superiority of BriMPR over the SOTA methods.

\section{Acknowledgments}
This work was supported by the Fundamental Research Funds for the Central Universities (No. 2025JBZX059), the Natural Science Foundation of Hebei Province (No. F2025105018), the Tangshan Municipal Science and Technology Plan Project (No.23130225E) and the Beijing Natural Science Foundation (No.4242046).

\bibliography{aaai2026}

\newpage
\clearpage
\appendix
\section{Appendix}

\subsection{Proof}
\label{proof}

\setcounter{theorem}{0}
\begin{theorem}
Given $x_1, \dots, x_n \in \mathbb{R}^d$ independently drawn from a multivariate normal distribution $\mathcal{N}(\mu, \Sigma)$, let $\hat{\Sigma}$ be the unbiased sample covariance matrix and $\hat{\sigma}^2 = [\hat{\sigma}_1^2, \dots, \hat{\sigma}_d^2]^T$ be the vector of its diagonal entries. Then, the mean squared errors satisfy:
\begin{equation}
\mathbb{E}\left[\|\hat{\Sigma} - \Sigma\|_F^2\right] = \mathcal{O}\left(\frac{d^2}{n}\right),
\mathbb{E}\left[\|\hat{\sigma}^2 - \sigma^2\|_2^2\right] = \mathcal{O}\left(\frac{d}{n}\right).
\end{equation}
\end{theorem}

\begin{proof}
The squared Frobenius norm of the error is:
\begin{equation}
\| \hat{\Sigma}-\Sigma\|_{F}^{2}=\sum_{i=1}^{d} \sum_{j=1}^{d} ( \hat{\Sigma}_{i j}-\Sigma_{i j} )^{2}.
\end{equation}
Taking expectations and using the unbiasedness of $\hat{\boldsymbol{\Sigma}}$ ($\mathbb{E} [ \hat{\boldsymbol{\Sigma}} ]=\boldsymbol{\Sigma}$):
\begin{equation}
\label{13}
\begin{aligned}
\mathbb{E} \left[ \| \hat{\Sigma}-\Sigma\|_{F}^{2} \right]&=
\sum_{i=1}^d \sum_{j=1}^d \mathbb{E}\left[(\hat{\Sigma}_{ij} - \Sigma_{ij})^2\right]\\
&=\sum_{i=1}^{d} \sum_{j=1}^{d} \mathrm{V a r} ( \hat{\Sigma}_{i j} ).
\end{aligned}
\end{equation}
Since $(n-1)\hat{\Sigma} \sim W_d(n-1, \Sigma)$, the Wishart distribution properties imply:
\begin{equation}
\text{Var}((n-1)\hat{\Sigma}_{ij}) = (n-1) (\Sigma_{ij}^2 + \Sigma_{ii} \Sigma_{jj}),
\end{equation}
and thus:
\begin{equation}
\label{15}
\text{Var}(\hat{\Sigma}_{ij}) = \frac{\Sigma_{ij}^2 + \Sigma_{ii} \Sigma_{jj}}{n-1}.
\end{equation}
Substituting Eq.~\eqref{15} back into Eq.~\eqref{13}:
\begin{equation}
\begin{aligned}
\mathbb{E}\left[\|\hat{\Sigma} - \Sigma\|_F^2\right] &= \frac{1}{n-1} \left( \sum_{i,j} \Sigma_{ij}^2 + \sum_{i,j} \Sigma_{ii} \Sigma_{jj} \right) \\
&= \frac{\| \Sigma\|_{F}^{2}+( \operatorname{t r} ( \Sigma) )^{2}} {n-1}.
\end{aligned}
\end{equation}
Assuming $\Sigma$ has bounded entries (independent of $n$ and $d$), $\| \Sigma\|_{F}^{2}=\mathcal{O} ( d^{2} )$ and $\mathrm{t r} ( \Sigma)={\cal O} ( d )$. Therefore:
\begin{equation*}
\mathbb{E}\left[\|\hat{\Sigma} - \Sigma\|_F^2\right] = \mathcal{O}\left(\frac{d^2}{n}\right).
\end{equation*}

The squared L2 norm of the error is:
\begin{equation}
\|\hat{\sigma}^2 - \sigma^2\|_2^2 = \sum_{i=1}^d (\hat{\sigma}_i^2 - \sigma_i^2)^2.
\end{equation}
Taking expectations and using $\mathbb{E} [ {\hat{\sigma}}_{i}^{2} ]=\sigma_{i}^{2}$:
\begin{equation}
\label{18}
\mathbb{E}\left[\|\hat{\sigma}^2 - \sigma^2\|_2^2\right] = \sum_{i=1}^d \mathbb{E}\left[(\hat{\sigma}_i^2 - \sigma_i^2)^2\right] = \sum_{i=1}^d \text{Var}(\hat{\sigma}_i^2).
\end{equation}
Since $\hat{\sigma}_i^2 = \hat{\Sigma}_{ii}$, from the Wishart distribution:
\begin{equation}
\label{19}
\text{Var}(\hat{\sigma}_i^2) = \text{Var}(\hat{\Sigma}_{ii}) = \frac{\Sigma_{ii}^2 + \Sigma_{ii} \Sigma_{ii}}{n-1} = \frac{2 \Sigma_{ii}^2}{n-1}.
\end{equation}
Substituting Eq.~\eqref{19} back into Eq.~\eqref{18}:
\begin{equation}
\mathbb{E}\left[\|\hat{\sigma}^2 - \sigma^2\|_2^2\right] = \sum_{i=1}^d \frac{2 \Sigma_{ii}^2}{n-1} = \frac{2}{n-1} \sum_{i=1}^d \Sigma_{ii}^2.
\end{equation}
Assuming $\Sigma_{ii}$ are bounded, $\sum_{i=1}^{d} \Sigma_{i i}^{2}={\cal O} ( d )$, yielding:
\begin{equation*}
\mathbb{E}\left[\|\hat{\sigma}^2 - \sigma^2\|_2^2\right] = \mathcal{O}\left(\frac{d}{n}\right).
\end{equation*}
\end{proof}

\subsection{Algorithm of BriMPR}
\label{code}

\begin{algorithm}[H]
\caption{BriMPR}
\label{alg:BriMPR}
\textbf{Input}: test data stream $(\mathbf{x}_t^{a}, \mathbf{x}_t^{v})$; modality-specific encoders $\Phi^a$ and $\Phi^v$ with prompts $\boldsymbol{P}^a = \{\boldsymbol{P}_0^a, \boldsymbol{P}_1^a, \dots, \boldsymbol{P}_{N - 1}^a\}$ and $\boldsymbol{P}^v = \{\boldsymbol{P}_0^v, \boldsymbol{P}_1^v, \dots, \boldsymbol{P}_{N - 1}^v\}$; joint module $\Psi$; classification head $h$; hyperparameters $\tau_0$, $D_0$, $\tau$.
\begin{algorithmic}[1] 
\FOR{$t = 1$ to $T$}
\STATE Get masked data $(\mathbf{x}^{a_m}, \mathbf{x}^{v_m})  \stackrel{\text{mask}}{\longleftarrow} (\mathbf{x}^{a}, \mathbf{x}^{v}) $.

\STATE Get fused features $\boldsymbol{Z}^{a_m v} = \Psi([\Phi^a(\mathbf{x}^{a_m});\Phi^v(\mathbf{x}^{v})])$, $\boldsymbol{Z}^{a v_m} = \Psi([\Phi^a(\mathbf{x}^{a});\Phi^v(\mathbf{x}^{v_m})])$ and $\boldsymbol{Z}^{a v} = \Psi([\Phi^a(\mathbf{x}^{a});\Phi^v(\mathbf{x}^{v})])$.
\STATE Get augmented predictions $\boldsymbol{y}^{a_mv}$, $\boldsymbol{y}^{av_m}$ via Eq.~\eqref{augpred}.

\STATE Calculate $\text{Disc}^a$, $\text{Disc}^v$ and $\text{Disc}^J$.
\STATE Calculate the adaptive temperature coefficient $\text{AdaTp} = 1 + \tau_{0} / (1 + \exp(D_0 - \operatorname{Disc}^{J}))$.
\STATE Get calibrated pseudo-labels $\hat{\boldsymbol{y}}^{av}$ via Eq.~\eqref{pseudolabel}.

\STATE Get unimodal features $\boldsymbol{Z}^{a}, \boldsymbol{Z}^{v}$ via Eq.~\eqref{unifeature}.

\STATE Calculate the overall loss $\cal{L}_{\text{BriMPR}}$ via Eq.~\eqref{overallloss}.
\STATE Update the modality-specific prompts $\boldsymbol{P}^a$ and $\boldsymbol{P}^v$.

\ENDFOR
\end{algorithmic}
\end{algorithm}

\subsection{Experimental Details}
\label{more_exp_details}

\begin{figure*}[t]
\centering
\begin{subfigure}{0.24\textwidth}
\includegraphics[width=\textwidth]{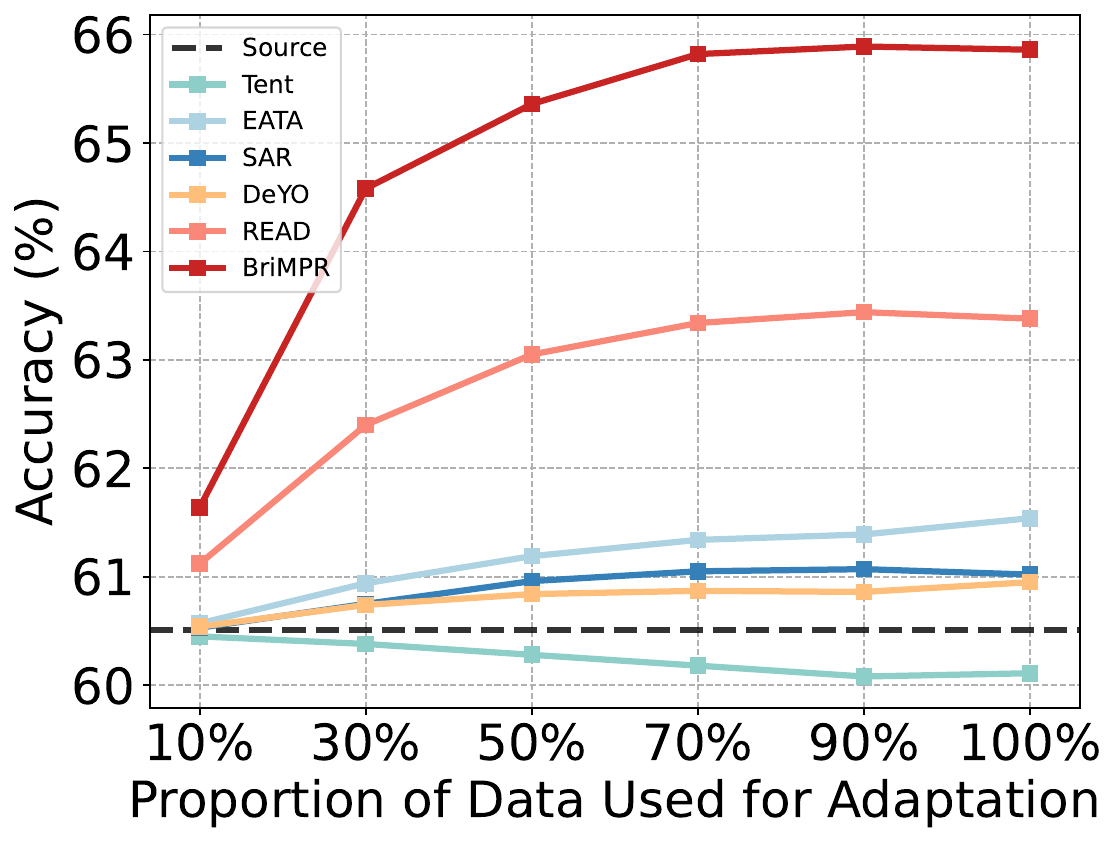}
\caption{Kinetics50-video}
\label{fig:limit-a}
\end{subfigure}
\hfill
\begin{subfigure}{0.24\textwidth}
\includegraphics[width=\textwidth]{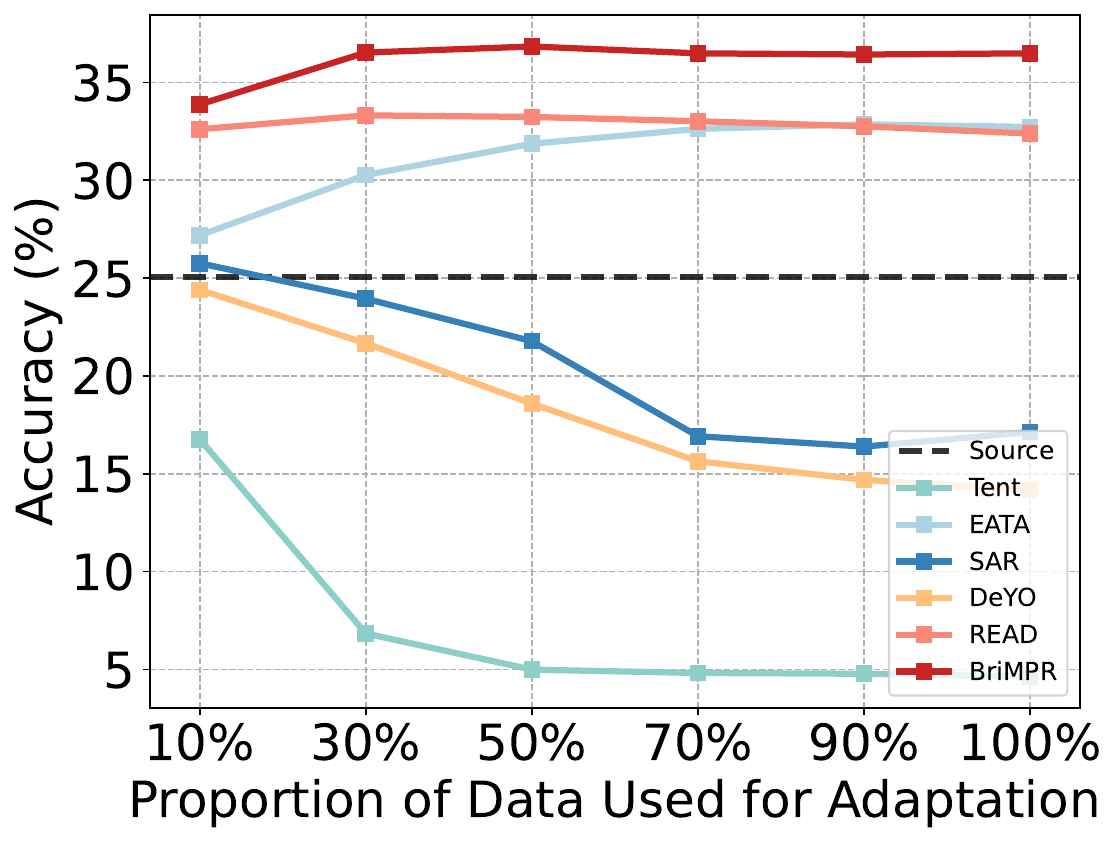}
\caption{VGGSound-audio}
\label{fig:limit-b}
\end{subfigure}
\hfill
\begin{subfigure}{0.24\textwidth}
\includegraphics[width=\textwidth]{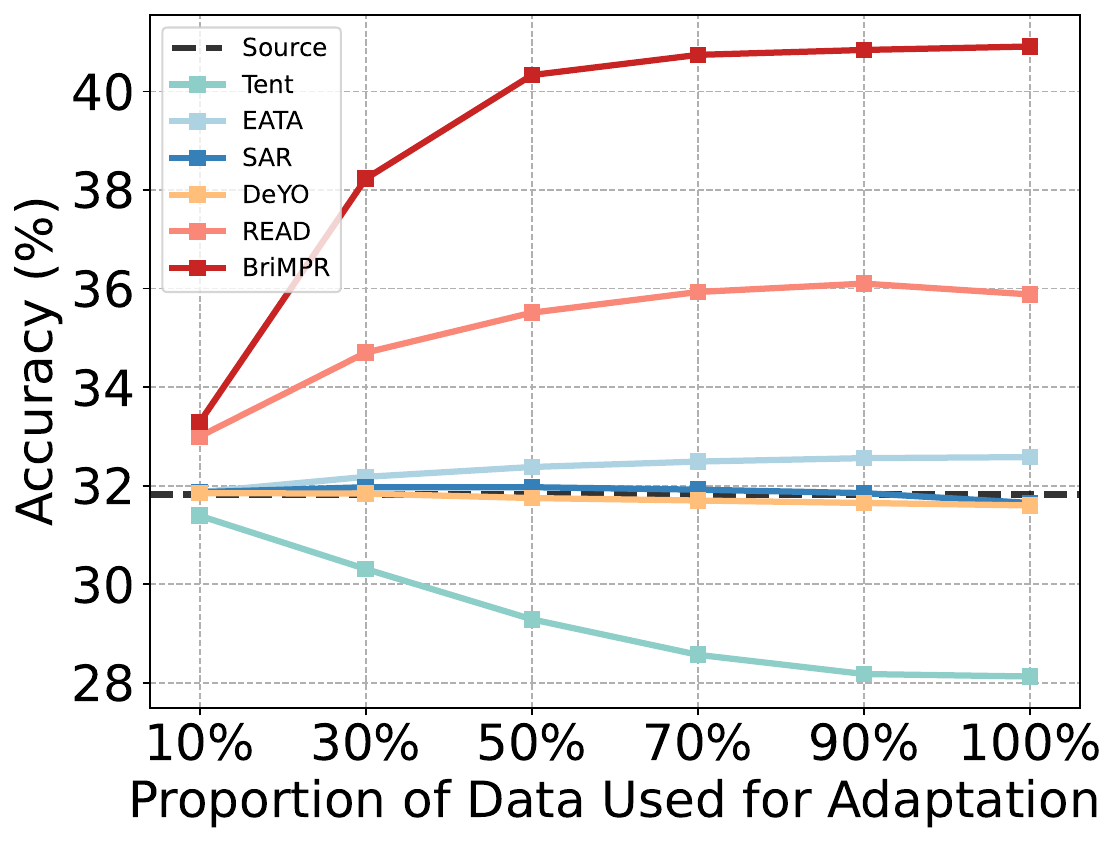}
\caption{Kinetics50-both}
\label{fig:limit-c}
\end{subfigure}
\hfill
\begin{subfigure}{0.24\textwidth}
\includegraphics[width=\textwidth]{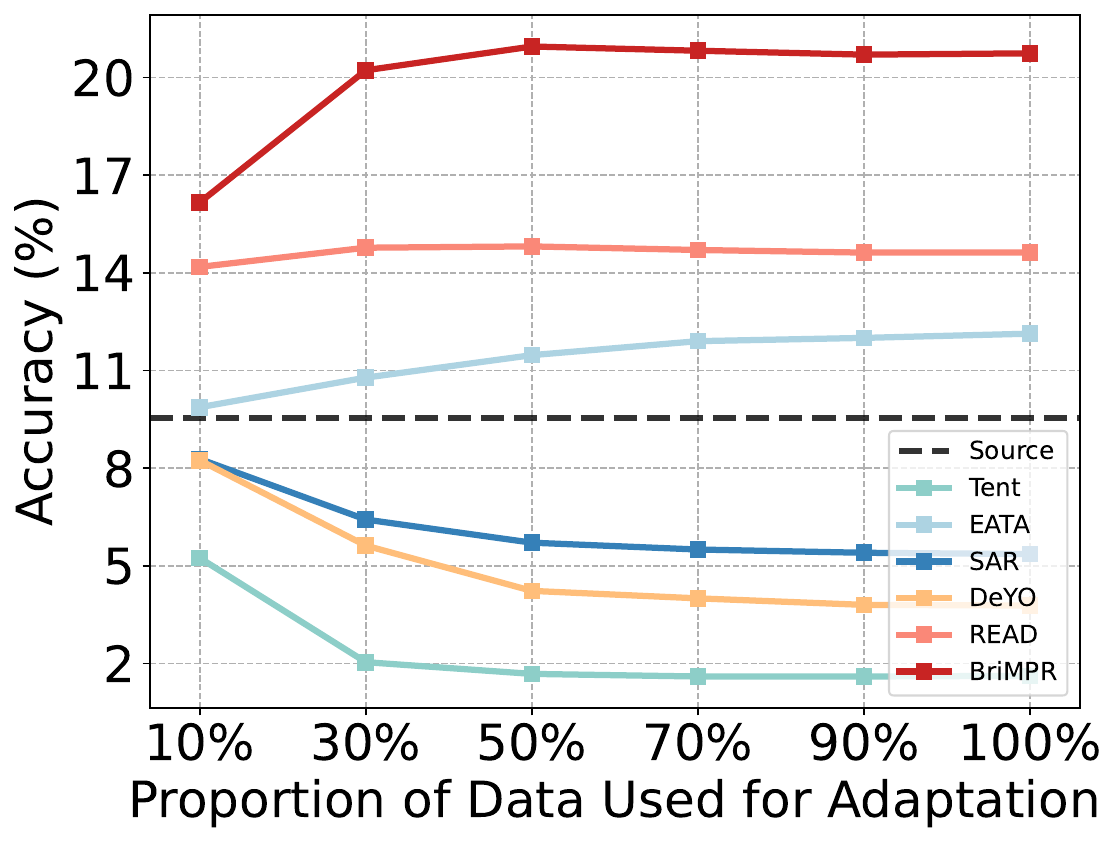}
\caption{VGGSound-both}
\label{fig:limit-d}
\end{subfigure}
\caption{Performance comparison when the data available for adaptation is limited under two tasks on Kinetics50-C and VGGSound-C. (a) and (b) correspond to the unimodal shift setting; (c) and (d) correspond to the multimodal shift setting.}
\label{fig:limit}
\end{figure*}


\begin{table*}[t]
\centering
\setlength{\tabcolsep}{1mm}
\small
\begin{tabular}{lcccccccccccccccc}
\toprule
 &
  \multicolumn{3}{c}{Noise} &
  \multicolumn{4}{c}{Blur} &
  \multicolumn{4}{c}{Weather} &
  \multicolumn{4}{c}{Digital} &
   \\ \cmidrule(lr){2-4}\cmidrule(lr){5-8}\cmidrule(lr){9-12}\cmidrule(lr){13-16}
Method &
  Gauss. &
  Shot &
  Impul. &
  Defoc. &
  Glass &
  Motion &
  Zoom &
  Snow &
  Frost &
  Fog &
  Bright. &
  Contr. &
  Elast. &
  Pixel. &
  Jpeg &
  \cellcolor[cmyk]{0.1012,0.0945,0,0}Avg. \\ \midrule
Source &
  62.7 &
  62.7 &
  62.4 &
  73.3 &
  70.6 &
  75.6 &
  71.3 &
  67.9 &
  65.7 &
  63.6 &
  79.5 &
  68.2 &
  73.6 &
  77.2 &
  75.2 &
  \cellcolor[cmyk]{0.1012,0.0945,0,0}70.0 \\
$\bullet$ Tent &
  63.1 &
  63.0 &
  62.9 &
  73.4 &
  71.3 &
  \underline{76.0} &
  72.1 &
  68.6 &
  66.6 &
  63.7 &
  79.3 &
  68.4 &
  74.3 &
  77.1 &
  75.0 &
  \cellcolor[cmyk]{0.1012,0.0945,0,0}70.3 \\
$\bullet$ EATA &
  63.3 &
  63.2 &
  63.1 &
  73.4 &
  71.5 &
  \underline{76.0} &
  72.2 &
  68.5 &
  66.8 &
  65.1 &
  79.4 &
  68.6 &
  74.3 &
  77.1 &
  75.2 &
  \cellcolor[cmyk]{0.1012,0.0945,0,0}70.5 \\
$\bullet$ SAR &
  63.2 &
  63.0 &
  62.8 &
  73.2 &
  71.5 &
  \textbf{76.1} &
  72.4 &
  68.5 &
  67.0 &
  65.1 &
  79.4 &
  68.7 &
  74.2 &
  77.1 &
  75.1 &
  \cellcolor[cmyk]{0.1012,0.0945,0,0}70.5 \\
$\bullet$ DeYO &
  63.3 &
  63.3 &
  63.2 &
  73.4 &
  71.4 &
  75.9 &
  72.2 &
  68.5 &
  66.7 &
  64.5 &
  79.3 &
  68.7 &
  74.2 &
  77.1 &
  75.3 &
  \cellcolor[cmyk]{0.1012,0.0945,0,0}70.5 \\
$\bullet$ READ$^\dag$ &
  \underline{64.0} &
  \underline{63.9} &
  \underline{64.1} &
  73.5 &
  \underline{72.1} &
  75.8 &
  \textbf{73.3} &
  \underline{69.5} &
  \underline{68.8} &
  \underline{68.9} &
  \underline{79.5} &
  \underline{69.8} &
  \underline{74.9} &
  77.3 &
  \underline{75.7} &
  \cellcolor[cmyk]{0.1012,0.0945,0,0}\underline{71.4} \\
$\bullet$ SuMi$^\dag$ &
  63.1 &
  63.2&
  63.1 &
  \underline{73.6} &
  71.5 &
  75.9 &
  72.1 &
  68.7 &
  66.7 &
  64.4 &
  \textbf{79.6} &
  68.7 &
  74.2 &
  \underline{77.4} &
  75.5 &
  \cellcolor[cmyk]{0.1012,0.0945,0,0}70.5 \\
\rowcolor[cmyk]{0.193,0,0.2222,0} 
$\bullet$ BriMPR$^\dag$ &
  \textbf{67.2} &
  \textbf{67.1} &
  \textbf{67.0} &
  \textbf{73.6} &
  \textbf{73.7} &
  75.8 &
  \underline{73.0} &
  \textbf{71.0} &
  \textbf{70.0} &
  \textbf{70.9} &
  79.2 &
  \textbf{71.1} &
  \textbf{75.6} &
  \textbf{77.8} &
  \textbf{76.1} &
  \textbf{72.6} \\ \bottomrule
\end{tabular}
\caption{Comparison with SOTA methods on Kinetics50-C under the unimodal shift setting (mixed severity levels of video corruption).}
\label{tab:ks50_mix_video}
\end{table*}

\begin{table}[t]
\centering
\setlength{\tabcolsep}{1mm}
\small
\begin{tabular}{lccccccc}
\toprule
 &
  \multicolumn{3}{c}{Noise} &
  \multicolumn{3}{c}{Weather} & \\ \cmidrule(lr){2-4} \cmidrule(lr){5-7}
Method &
  Gauss. &
  Traff. &
  Crowd &
  Rain &
  Thund. &
  Wind &
  \cellcolor[cmyk]{0.1012,0.0945,0,0}Avg. \\ \midrule
Source &
  76.7 &
  65.0 &
  68.5 &
  69.2 &
  68.9 &
  72.1 &
  \cellcolor[cmyk]{0.1012,0.0945,0,0}70.1 \\
$\bullet$ Tent &
  76.8 &
  67.1 &
  69.9 &
  70.9 &
  69.3 &
  \underline{72.8} &
  \cellcolor[cmyk]{0.1012,0.0945,0,0}71.1 \\
$\bullet$ EATA &
  \underline{77.0} &
  67.4 &
  70.1 &
  \underline{71.4} &
  70.6 &
  72.5 &
  \cellcolor[cmyk]{0.1012,0.0945,0,0}71.5 \\
$\bullet$ SAR &
  76.8 &
  67.0 &
  69.8 &
  70.8 &
  70.2 &
  72.5 &
  \cellcolor[cmyk]{0.1012,0.0945,0,0}71.2 \\
$\bullet$ DeYO &
  76.9 &
  67.2 &
  70.0 &
  71.2 &
  70.0 &
  72.6 &
  \cellcolor[cmyk]{0.1012,0.0945,0,0}71.3 \\
$\bullet$ READ$^\dag$ &
  \textbf{77.1} &
  \underline{70.7} &
  \underline{71.3} &
  \textbf{72.5} &
  \underline{73.0} &
  72.5 &
  \cellcolor[cmyk]{0.1012,0.0945,0,0}\underline{72.8} \\
$\bullet$ SuMi$^\dag$ &
  76.8 &
  66.0 &
  69.4 &
  70.9 &
  70.1 &
  72.1 &
  \cellcolor[cmyk]{0.1012,0.0945,0,0}70.9 \\
\rowcolor[cmyk]{0.193,0,0.2222,0} 
$\bullet$ BriMPR$^\dag$ &
  \textbf{77.1} &
  \textbf{70.8} &
  \textbf{72.6} &
  \textbf{72.5} &
  \textbf{73.4} &
  \textbf{73.0} &
  \textbf{73.2} \\ \bottomrule
\end{tabular}
\caption{Comparison with SOTA methods on Kinetics50-C under the unimodal shift setting (mixed severity levels of audio corruption).}
\label{tab:ks50_mix_audio}
\end{table}

\begin{table*}[t]
\centering
\setlength{\tabcolsep}{1mm}
\small
\begin{tabular}{lcccccccccccccccc}
\toprule
 &
  \multicolumn{3}{c}{Noise} &
  \multicolumn{4}{c}{Blur} &
  \multicolumn{4}{c}{Weather} &
  \multicolumn{4}{c}{Digital} &
   \\ \cmidrule(lr){2-4}\cmidrule(lr){5-8}\cmidrule(lr){9-12}\cmidrule(lr){13-16}
Method &
  Gauss. &
  Shot &
  Impul. &
  Defoc. &
  Glass &
  Motion &
  Zoom &
  Snow &
  Frost &
  Fog &
  Bright. &
  Contr. &
  Elast. &
  Pixel. &
  Jpeg &
  \cellcolor[cmyk]{0.1012,0.0945,0,0}Avg. \\ \midrule
Source &
  35.3 &
  35.0 &
  34.6 &
  50.9 &
  51.6 &
  56.7 &
  51.0 &
  42.2 &
  42.9 &
  38.8 &
  63.1 &
  43.3 &
  56.1 &
  58.2 &
  56.2 &
  \cellcolor[cmyk]{0.1012,0.0945,0,0}47.7 \\
$\bullet$ Tent &
  30.4 &
  29.9 &
  29.8 &
  47.1 &
  50.8 &
  55.9 &
  50.1 &
  38.9 &
  41.2 &
  32.5 &
  63.3 &
  39.1 &
  56.7 &
  56.8 &
  54.7 &
  \cellcolor[cmyk]{0.1012,0.0945,0,0}45.2 \\
$\bullet$ EATA &
  36.5 &
  36.3 &
  36.1 &
  52.7 &
  53.5 &
  58.7 &
  52.9 &
  43.6 &
  44.2 &
  39.3 &
  64.6 &
  44.2 &
  57.6 &
  59.9 &
  57.8 &
  \cellcolor[cmyk]{0.1012,0.0945,0,0}49.2 \\
$\bullet$ SAR &
  35.2 &
  34.9 &
  34.4 &
  52.1 &
  53.0 &
  58.1 &
  52.9 &
  43.4 &
  44.1 &
  37.1 &
  63.8 &
  42.4 &
  57.1 &
  59.1 &
  56.9 &
  \cellcolor[cmyk]{0.1012,0.0945,0,0}48.3 \\
$\bullet$ DeYO &
  34.5 &
  33.9 &
  33.7 &
  52.1 &
  53.1 &
  58.4 &
  52.7 &
  43.0 &
  44.2 &
  37.4 &
  64.3 &
  42.8 &
  57.5 &
  59.2 &
  56.8 &
  \cellcolor[cmyk]{0.1012,0.0945,0,0}48.2 \\
$\bullet$ READ$^\dag$ &
  \underline{39.4} &
  \underline{39.3} &
  \underline{38.9} &
  \underline{57.3} &
  \underline{56.0} &
  \textbf{61.6} &
  \underline{56.0} &
  \underline{47.2} &
  \underline{48.5} &
  \underline{46.2} &
  \underline{66.5} &
  \underline{47.8} &
  \underline{59.6} &
  \underline{62.6} &
  \underline{60.3} &
  \cellcolor[cmyk]{0.1012,0.0945,0,0}\underline{52.5} \\
$\bullet$ SuMi$^\dag$ &
  35.3 &
  35.1 &
  34.8 &
  51.5 &
  52.5 &
  57.7 &
  52.0 &
  42.6 &
  43.5 &
  37.9 &
  63.9 &
  43.2 &
  56.9 &
  58.8 &
  56.7 &
  \cellcolor[cmyk]{0.1012,0.0945,0,0}48.2 \\
\rowcolor[cmyk]{0.193,0,0.2222,0} 
$\bullet$ BriMPR$^\dag$ &
  \textbf{45.2} &
  \textbf{44.5} &
  \textbf{44.4} &
  \underline{56.5} &
  \textbf{57.8} &
  \underline{60.6} &
  \textbf{57.2} &
  \textbf{49.4} &
  \textbf{50.3} &
  \textbf{51.1} &
  \textbf{66.8} &
  \textbf{51.8} &
  \textbf{61.2} &
  \textbf{64.4} &
  \textbf{62.4} &
  \textbf{54.9} \\ \bottomrule
\end{tabular}
\caption{Comparison with SOTA methods on Kinetics50-C under the multimodal shift setting (mixed severity levels).}
\label{tab:ks50_mix_both}
\end{table*}

\begin{figure*}[t]
\centering
\begin{subfigure}{0.85\textwidth}
\includegraphics[width=\textwidth]{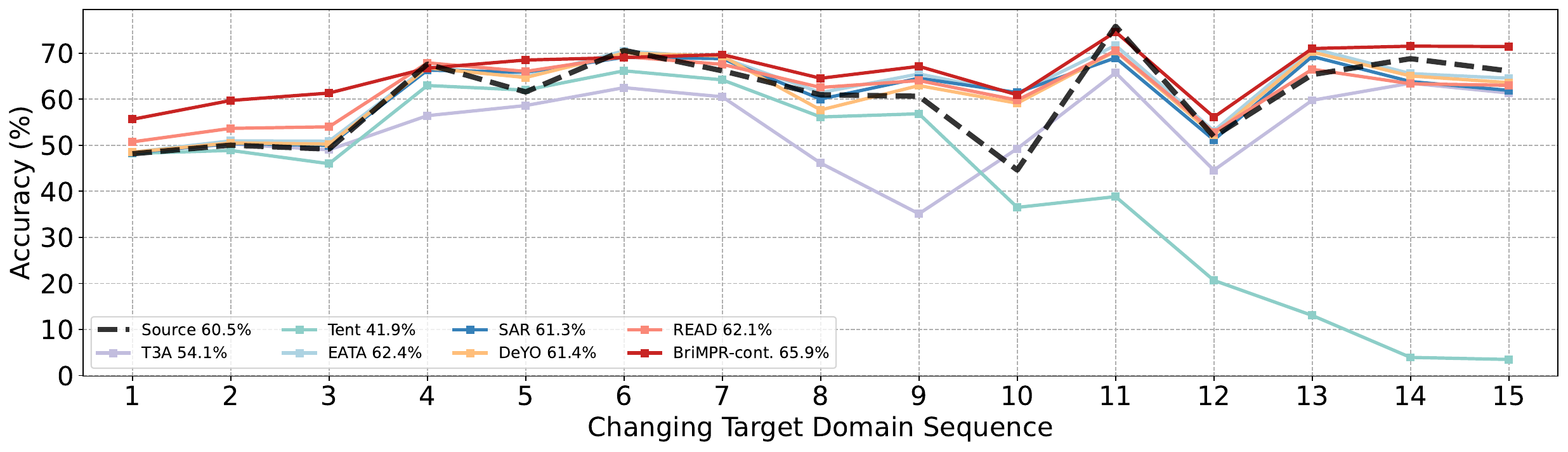}
\caption{Kinetics50-video}
\label{fig:cmmtta_ks50_video}
\end{subfigure}
\hfill
\begin{subfigure}{0.85\textwidth}
\includegraphics[width=\textwidth]{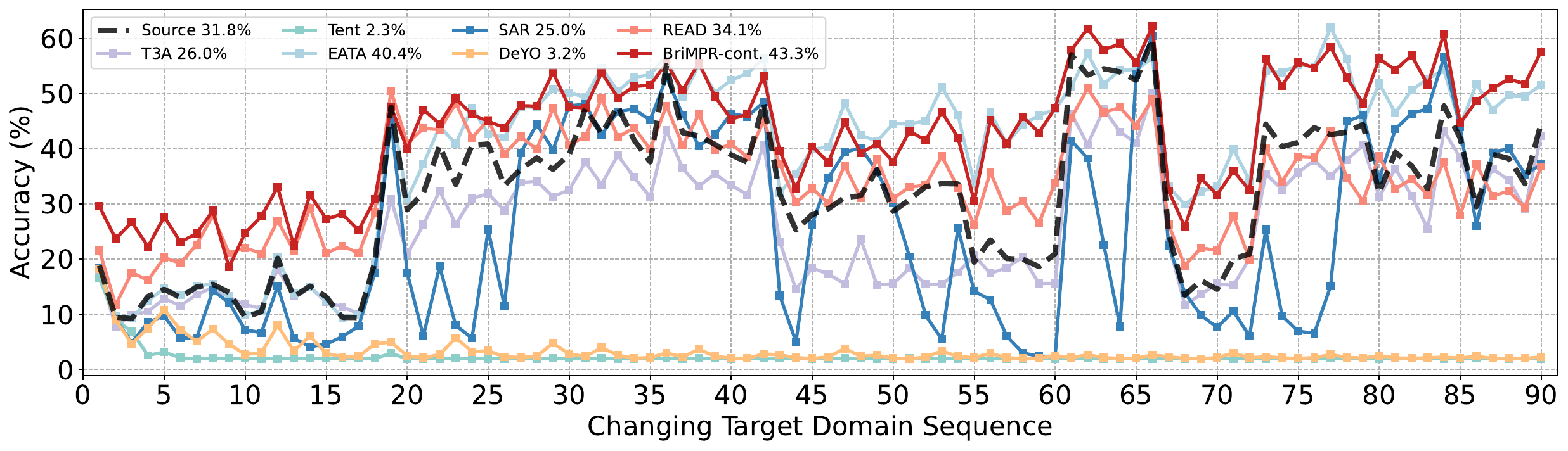}
\caption{Kinetics50-both}
\label{fig:cmmtta_ks50_both}
\end{subfigure}
\caption{Comparison with the state-of-the-arts on Kinetics50-C under the continual setting (severity level 5). ``Kinetics50-video'' contains 15 continuous domains, while ``Kinetics50-both'' contains $15 \times 6 = 90$ continuous domains. The legend shows the average accuracy across all domains.}
\label{fig:cmmtta}
\end{figure*}

\subsubsection{Details of Baselines.}
\label{baseline}

\begin{itemize}
\item \textbf{Tent}~\cite{Tent} optimizes the affine parameters in the normalization layer by minimizing the entropy of the model's predictions. Because entropy can reflect the uncertainty of the predictions, and the normalization layer is associated with the distribution information. 
In the transformer-based CAV-MAE, we replace the Batch Normalization (BN) layers in the original implementation with Layer Normalization (LN) layers.
\item \textbf{EATA}~\cite{EATA} selects reliable and on-redundant test samples with low entropy to participate in entropy minimization, and introduces a weighted Fisher regularizer to prevent significant changes in the parameters that are crucial for the in-distribution data during adaptation, thus alleviating the forgetting problem. 
The entropy threshold $E_0$ is set to $0.4\times\ln C$ (where $C$ is the number of classes). The cosine similarity threshold $\epsilon$ used for filtering redundant samples is set to 0.1. The trade-off hyperparameter $\beta$ is set to 1. The Fisher information is calculated using 2,000 unlabeled in-distribution samples, and the moving average factor $\alpha$ is set to 0.1.
\item \textbf{SAR}~\cite{SAR} introduces sharpness-aware learning and minimizes entropy, optimizing the model weights to a flat minimum to enhance the robustness against noisy samples. 
Similar to EATA, the entropy threshold $E_0$ is set to $0.4\times\ln C$. The radius $\rho$ in the sharpness-aware optimization is set to 0.05. The model recovery threshold $e_0$ is set to 0.2. The moving average factor used to track the loss value is set to 0.9.
\item \textbf{DeYO}~\cite{DeYO} introduces the Pseudo-Label Probability Difference (PLPD) metric to identify harmful samples that cannot be detected by entropy. It quantifies the impact of object shape information on predictions by measuring the difference in predictions before and after applying a single image transformation. 
The entropy threshold $\tau_{Ent}$ is set to $0.5\times\ln C$, the PLPD threshold $\tau_{PLPD}$ is set to $0.2\times\ln C$, and the normalization factor $Ent_0$ in the weighting function is set to $0.4\times\ln C$. The default patch-shuffling is used as the image transformation. 
\item \textbf{FOA}~\cite{FOA} inserts prompts at the input level and employs the derivative-free covariance matrix adaptation (CMA) evolution strategy to learn the prompts. Additionally, it further adjusts the activation features at the output feature level to align them with the source domain. 
The number of prompt embeddings $N_p$ is set to 1 with the default uniform initialization. The population size $K$ in the CMA evolution strategy is set to $27 = 4 + 3 \ln(2 \times 768)$. The trade-off parameter $\lambda$ in the fitness function is set to $0.4$. The step size $\gamma$ in the back-to-source activation shifting is set to 1.0. The moving average factor for computing the test statistics is set to 0.1.
\item \textbf{READ}~\cite{READ} updates the self-attention layer of the fusion module by optimizing the confidence-aware loss function, which promotes reliable fusion across different modalities. 
The confidence threshold $\gamma$ is set to $e^{-1}$.
\item \textbf{ABPEM}~\cite{ABPEM} aligns cross-attention to self-attention to reduce inter-modal differences, while excluding non-dominant class samples to reduce gradient noise in entropy loss.
The threshold $k$ for class ranking is set to 8/30 for Kinetics50-C and VGGSound-C. The weight $\lambda$ for the attention bootstrapping loss is set to 1.
\item \textbf{SuMi}~\cite{SuMi} progressively adapts to strongly out-of-distribution samples through interquartile range smoothing, selects high-quality samples via unimodal-assisted identification, and balances adaptation across modalities by leveraging mutual information sharing. 
The multimodal threshold $\gamma_m$ and the normalization factor $\text{Ent}_0$ are set to $0.4\times \ln C$. The unimodal threshold $\gamma_u$ is set to $e^{-1}$. The smoothing coefficient $\beta$ is set to 0.6/0.9, the weighting term $\lambda$ is set to 5.0 and the unimodal assistance $t$ is set to 1.0 by default for Kinetics50-C and VGGSound-C. For multimodal shift setting and real-world shift setting, we set the mutual information sharing term $t_{0}$ as $iter/2$.
\end{itemize}

\subsection{Further Experiments}
\label{further_experiments}

\subsubsection{Limited data for adaptation.}
To explore the adaptation process of the proposed BriMPR, we further conduct experiments in the scenario where the data available for adaptation is limited. 
In this case, the model can only utilize a portion of the test data for adaptation in the initial stage and then make predictions on the remaining test data. This is because it is unrealistic for an intelligent perception system deployed in the real world to maintain updates all the time. 
An excellent multimodal test-time adaptation algorithm should still demonstrate good performance improvements even with limited target data and benefit from an increase in the available data. 
As shown in Fig.~\ref{fig:limit}, BriMPR consistently maintains the best performance in all scenarios, demonstrating its data efficiency. Some methods even exhibit performance degradation with increased available test data (e.g., Tent, SAR, and DeYO in Fig.~\ref{fig:limit-b}), reflecting that the error accumulation~\cite{error} commonly faced by unimodal test-time adaptation is further exacerbated in multimodal scenarios.

\subsubsection{Results under Mixed Severity Levels.}
Our main experimental results are obtained with the largest corruption severity level 5. We conduct experiments under both two settings with the mixed severity levels on Kinetics50-C to further verify the robustness of our BriMPR. In this case, the severity level of the corrupted modality is randomly selected from 1 to 5. For each experiment, we generate three sets of test data with three random seeds. 
As shown in Tab.~\ref{tab:ks50_mix_video}, Tab.~\ref{tab:ks50_mix_audio}, and Tab.~\ref{tab:ks50_mix_both}, BriMPR significantly outperforms other methods in the vast majority of cases.

\subsubsection{Results of Continual Multimodal Test-Time Adaptation.}
We further introduce continuously changinging domains~\cite{CoTTA, RMT, VDP} and refer to this setting as Continual Multimodal Test-Time Adaptation (CMMTTA), which is more challenging because the model lacks domain labels during adaptation.

In the continual unimodal shift setting, only the corrupted modality undergoes distribution changes. In contrast, the continual multimodal shift setting involves distribution changes in one modality at a time. Taking ``Kinetics50-both'' as an example, each domain is represented as a combination of ``video corruption + audio corruption'', where the modality changed is highlighted. Then the domain sequence is constructed as follows: 
Gauss. + Gauss. $\rightarrow$ Gauss. + \textbf{Traff.} $\rightarrow$ \dots $\rightarrow$ Gauss. + \textbf{Wind} $\rightarrow$ \textbf{Shot.} + Wind $\rightarrow$ Shot. + \textbf{Thund.} $\rightarrow$ \dots $\rightarrow$ Shot. + \textbf{Gauss.} $\rightarrow$ \textbf{Impul.} + Gauss. $\rightarrow$ \dots, resulting in a total of 15 × 6 = 90 continuous domains.

To adapt BriMPR to the CMMTTA setting, we propose a simple variant, BriMPR-continual, which leverages the by-product unimodal distribution discrepancy $\text{Disc}^u$ to detect domain shifts. 
Specifically, we maintain a sliding window of size $w$ (e.g., $w=10$) to store the most recent $\text{Disc}^u$ values. For each new $\text{Disc}^u_t$, we calculate the Z-score with respect to the mean $\mu_t^u$ and standard deviation $\sigma_t^u$ of the windowed values as: $Z_t^u=\frac{\text{Disc}^u_t - \mu_t^u}{\sigma_t^u}$. If $Z_t^u > k$ (e.g., $k=5$), we interpret it as a significant domain shift for modality $u$, and accordingly re-initialize the modality-specific prompts $\boldsymbol{P}^u$. 
As shown in Fig.~\ref{fig:cmmtta}, all compared methods exhibit some degree of knowledge forgetting during continual adaptation. Among them, EATA~\cite{EATA} performs well after 20 domains due to its specially designed anti-forgetting mechanism. In contrast, our BriMPR-continual consistently improves over the Source baseline and achieves the best performance.

\subsubsection{Efficiency comparisons.}
Tab.~\ref{tab:efficiency} compares the computational efficiency of different methods on VGGSound-C. Although BriMPR involves data augmentation and requires additional forward passes, the augmentation is performed via masking, making it more efficient than other augmentation-based methods like DeYO. Furthermore, thanks to the parameter efficiency of prompt tuning, BriMPR introduces fewer learnable parameters.

\begin{table*}[t]
\centering
\begin{tabular}{l|cccccccc}
\toprule
Method & Source & Tent & EATA & SAR & DeYO & READ$^\dag$ & SuMi$^\dag$ & BriMPR$^\dag$ \\ \midrule
Time (s) & 66.8 & 161.1 & 169.8 & 236.6 & 220.4 & 135.2 & 424.4 & 186.2 \\
\# Params (M) & 0 & 0.218 & 0.218 & 0.218 & 0.218 & 1.772 & 0.218 & 0.169 \\ \bottomrule
\end{tabular}
\caption{Efficiency comparisons among different methods on VGGSound-C.}
\label{tab:efficiency}
\end{table*}

\subsection{More Ablation Studies}

\subsubsection{Prompts are Better Distribution Calibrators.}
\label{prompt_vs_ln}
In many existing TTA works~\cite{AdaBN, Tent, EATA}, optimizing only the parameters of the normalization layer is regarded as a shortcut for calibrating the target feature distribution. In contrast, BriMPR calibrates the unimodal target feature distribution by optimizing the embedded prompts at each layer of the modality-specific encoders. As illustrated in Fig.~\ref{fig:prompt_vs_ln}, when using the same loss function ${\cal L}_{\mathrm{PMGFA}}$, the variant ${\cal L}_{\mathrm{PMGFA}-\text{LN}}$ that optimizes the LayerNorm parameters consistently underperforms prompt optimization across all tasks, while also requiring a larger number of trainable parameters.

\begin{figure}[t]
\centering
\includegraphics[width=1\columnwidth]{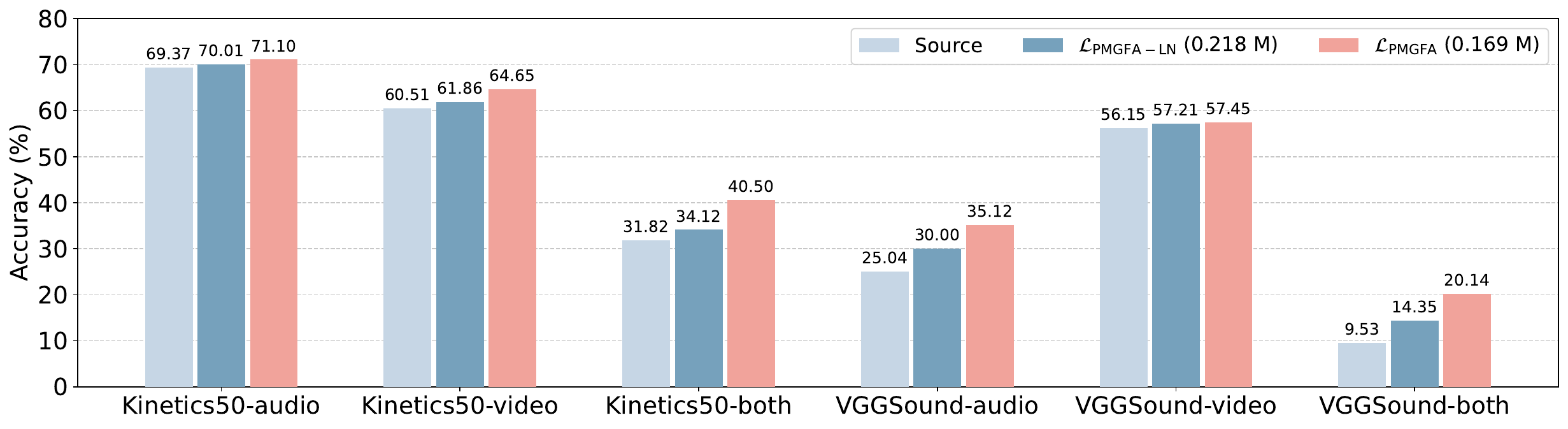}
\caption{Comparison between updating LN parameters and updating prompts.}
\label{fig:prompt_vs_ln}
\end{figure}

\subsubsection{Impact of $\tau_0$, $D_0$ and $\tau$.}
We analyze the hyperparameters $\tau_0$, $D_0$ (in \text{AdaTP}) and $\tau$ (in Eq.~\eqref{iicl}) on Kinetics50-C. As shown in Fig.~\ref{fig:ab}, $\tau_0$ and $D_0$ are insensitive within reasonable ranges under both unimodal and multimodal corruption. This allows \text{AdaTP} to adaptively set temperatures and mitigate overconfident pseudo-labels. Fig.~\ref{fig:tau} shows that our method performs stably under different $\tau$.

\begin{figure*}[t]
\centering

\begin{subfigure}[b]{0.48\textwidth}
\includegraphics[width=0.9\textwidth]{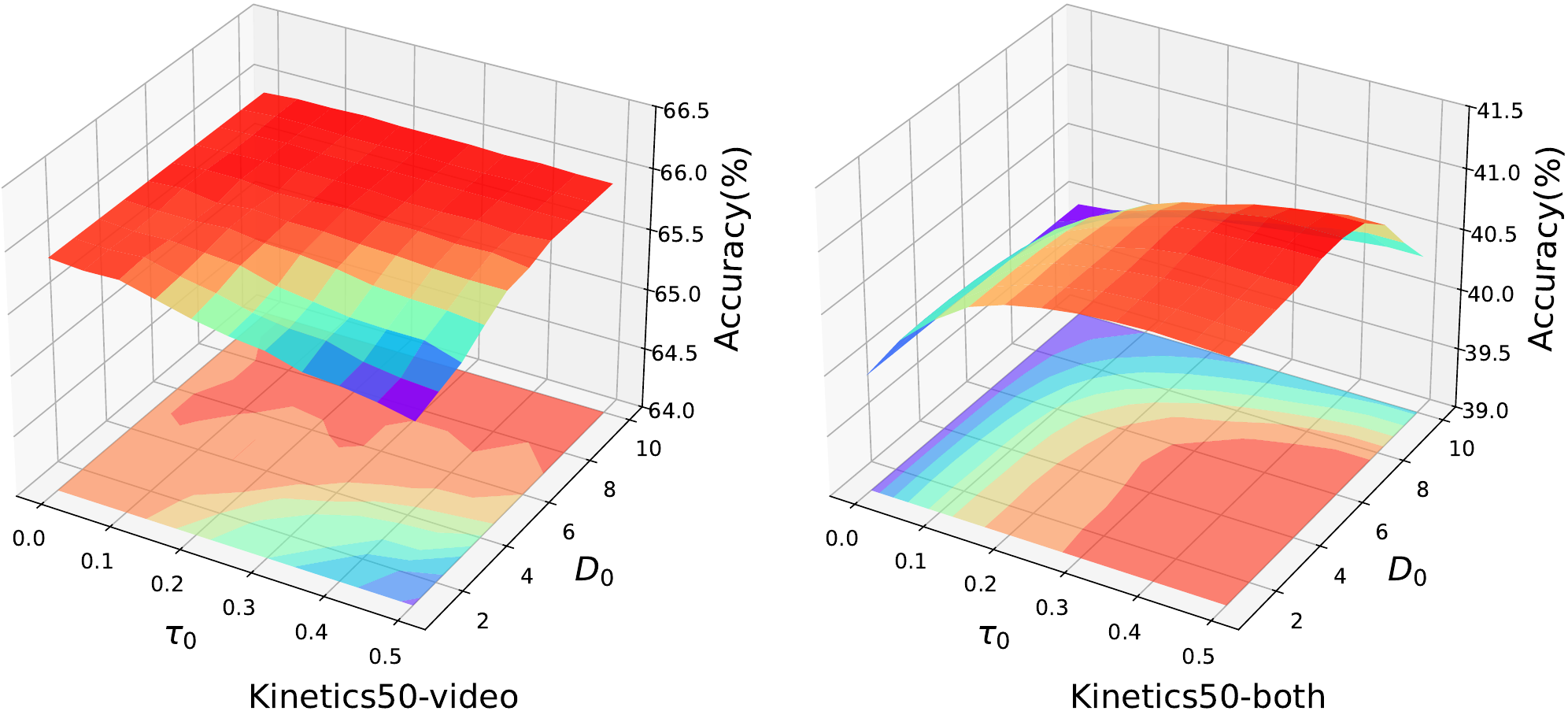}
\caption{$\tau_0$, $D_0$}
\label{fig:ab}
\end{subfigure}
\begin{subfigure}[b]{0.45\textwidth}
\includegraphics[width=0.9\textwidth]{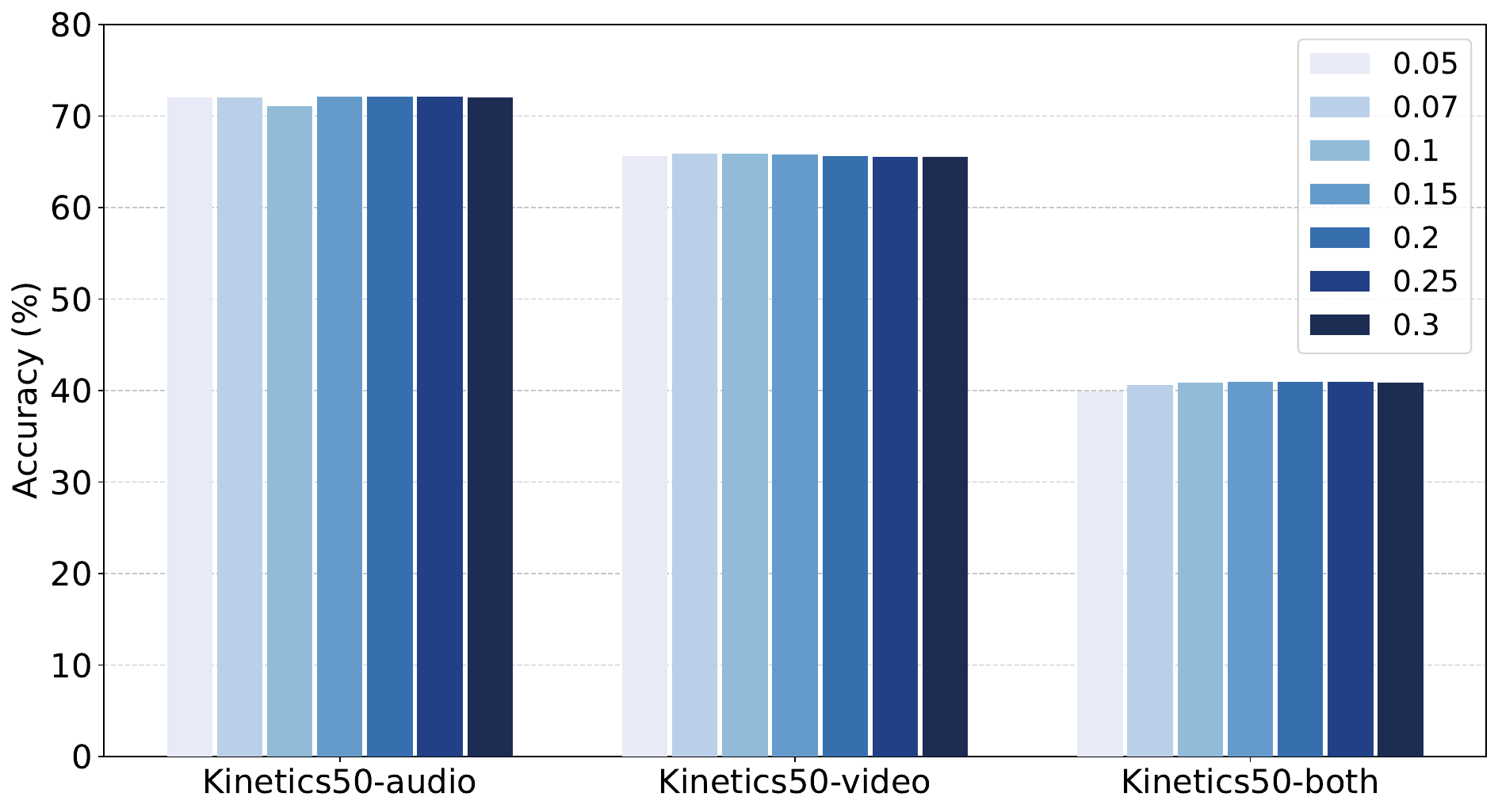}
\caption{$\tau$}
\label{fig:tau}
\end{subfigure}

\caption{Analysis of $\tau_0$, $D_0$ and $\tau$ on Kinetics50-C.}
\label{fig:hyper1}
\end{figure*}

\subsubsection{Impact of Number of Prompts.}
In Tab.~\ref{tab:prompt_number}, we present the results for different numbers of prompts, with 10 being the default value adopted by BriMPR. Previous prompt tuning work~\cite{VPT} indicates that the optimal number of prompts varies across different tasks. In our experiments on multimodal test-time adaptation task, adding more learnable prompts generally leads to performance improvements. Nevertheless, thanks to the parameter efficiency of prompt tuning, even a small number of prompts proves highly competitive, relieving us from the need to exhaustively search for the optimal prompt count.

\begin{table*}[t]
\centering
\setlength{\tabcolsep}{1mm}
\small
\begin{tabular}{cccccccc}
\toprule
Num. of Prompts & Ks50-audio & Ks50-video & Ks50-both & VGG-audio & VGG-video & VGG-both & Params (M) \\ \midrule
1  & 71.22 & 63.80 & 37.95 & 35.58 & 57.62 & 19.05 & 0.017 \\
3  & 71.48 & 64.55 & 39.24 & 36.31 & 57.74 & 20.30 & 0.051 \\
5  & 71.73 & 64.91 & 39.94 & 36.33 & 57.69 & 20.48 & 0.084 \\
\rowcolor[cmyk]{0.193,0,0.2222,0} 
10 & 72.01 & 65.86 & 40.91 & 36.49 & 57.68 & 20.74 & 0.169 \\
20 & 72.19 & 66.07 & 41.90 & 36.50 & 57.69 & 20.98 & 0.338 \\ \bottomrule
\end{tabular}
\caption{Comparison of different numbers of prompts.}
\label{tab:prompt_number}
\end{table*}

\subsubsection{Impact of Number of Unlabeled Source Samples.}
By avoiding the use of covariance matrices in moment matching and adopting a non-squared formulation, $\mathcal{L}_{\text{PMGFA}}$ demonstrates greater robustness and significantly improves the quality of unimodal distribution alignment. In all the experiments, we pre-estimate the source statistics $\{\hat{\boldsymbol{\mu}}_{i}^{s, u}, \hat{\boldsymbol{\sigma}}_{i}^{s, u}\}_{i = 1}^N(u\in\{a, v\})$ using 32 unlabeled source data. As shown in Fig.~\ref{fig:source_number}, BriMPR performs stably under different amounts of available source data.

\begin{figure}[t]
\centering
\includegraphics[width=1\columnwidth]{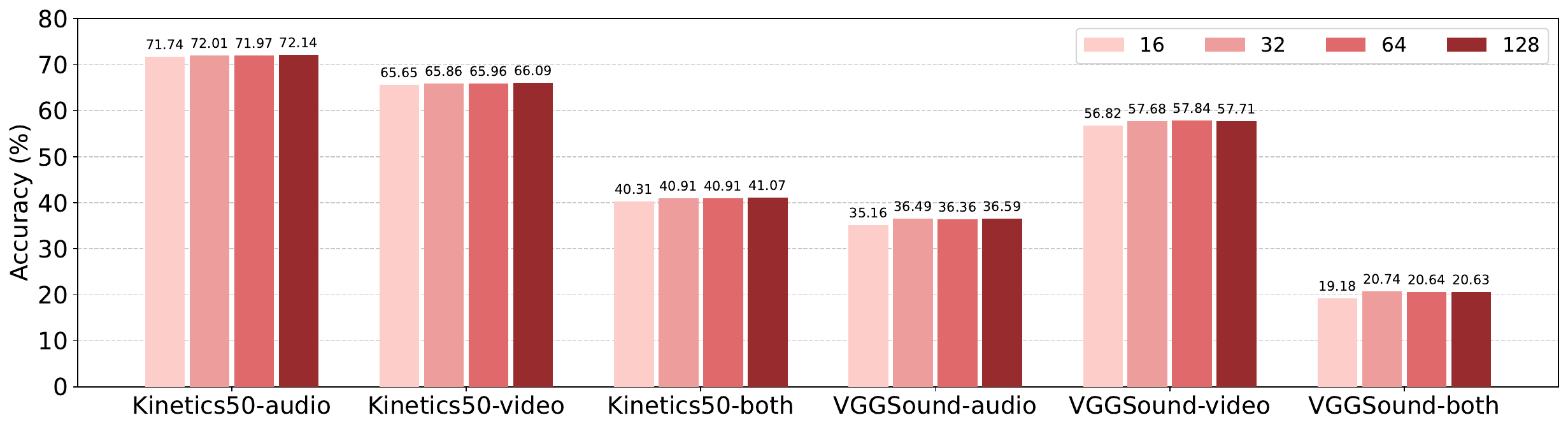}
\caption{Comparison of different numbers of unlabeled source samples.}
\label{fig:source_number}
\end{figure}

\subsubsection{Impact of Mask Ratio.}
In Fig.~\ref{fig:mask_ratio}, we further investigate the sensitivity of BriMPR to the mask ratio in $\mathcal{L}_\text{CMER}$, traversing the values of [0.3, 0.4, 0.5, 0.6, 0.7] around the default value of 0.5. Our method maintains stable performance across different mask ratios. Meanwhile, it can be observed that: under the unimodal shift setting, increasing the mask ratio tends to improve performance (e.g., VGGSound-audio: 0.3 $\rightarrow$ 0.7, accuracy 35.91\% $\rightarrow$ 36.97\%); whereas under the multimodal shift setting, an excessively high mask ratio degrades performance (e.g., Kinetics50-both: 0.3 $\rightarrow$ 0.7, accuracy 41.08\% $\rightarrow$ 40.19\%). This is because under the former setting, there exists a clean modality, and boldly discarding the information of this modality promotes the recovery of the corrupted modality; whereas under the latter setting, the severity of shift varies among each modality, and excessive masking may suppress useful information and introduce noise from unreliable modalities.

\begin{figure}[t]
\centering
\includegraphics[width=1\columnwidth]{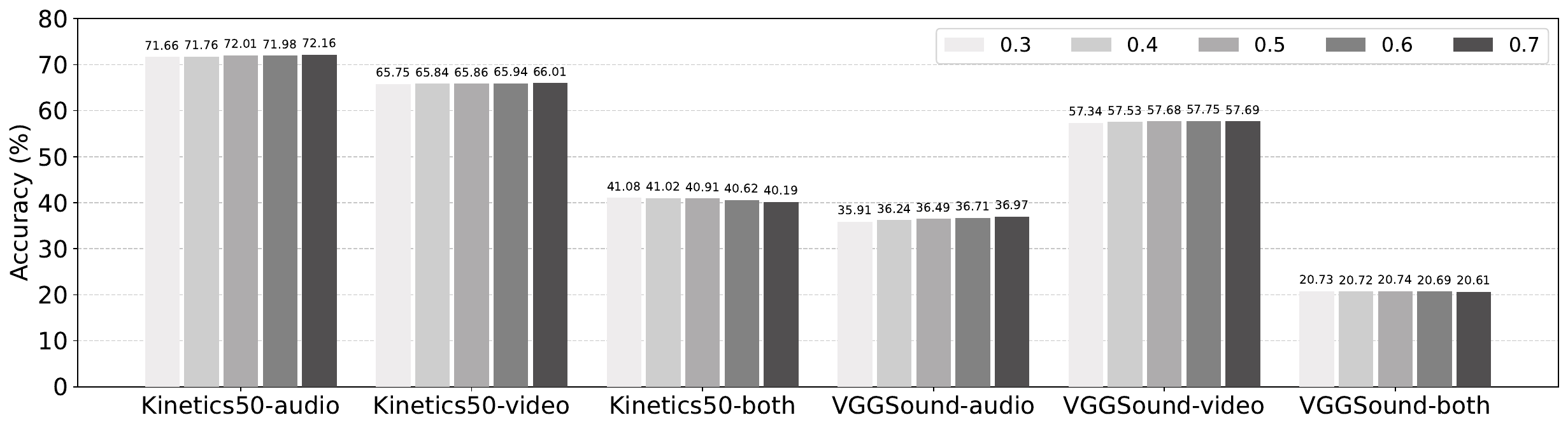}
\caption{Comparison of different mask ratios.}
\label{fig:mask_ratio}
\end{figure}

\end{document}